\documentclass{article}

\usepackage[final]{neurips_2020}
\usepackage[utf8]{inputenc} 
\usepackage[T1]{fontenc}    
\usepackage{hyperref}       
\usepackage{url}            
\usepackage{booktabs}       
\usepackage{amsfonts}       
\usepackage{nicefrac}       
\usepackage{microtype}      
\usepackage[thinc]{esdiff}
\usepackage{algorithm}
\usepackage{algorithmic}
\usepackage{graphicx}
\usepackage{subfigure}

\usepackage{amsthm, amssymb}
\theoremstyle{definition}

\newtheorem{theorem}{Theorem}
\newtheorem{lemma}{Lemma}
\newtheorem{proposition}{Proposition}

\usepackage{mathtools}
\DeclarePairedDelimiter\ceil{\lceil}{\rceil}

\usepackage{natbib}
\title{Sub-linear Regret Bounds for Bayesian Optimisation in Unknown Search Spaces}

\author{%
  Hung Tran-The\thanks{Correspondence to: Hung Tran-The <\texttt{hung.tranthe@deakin.edu.au}>.}, Sunil Gupta, Santu Rana, Huong Ha, Svetha Venkatesh \\
  Applied Artificial Intelligence Institute\\
  Deakin University, Australia\\
}

\begin{document}

\maketitle

\begin{abstract}
Bayesian optimisation is a popular method for  efficient optimisation of expensive black-box functions. Traditionally, BO assumes that the search space is known. However, in  many problems, this assumption does not hold. To this end, we propose a novel BO algorithm which expands (and shifts) the search space over iterations based on controlling the expansion rate thought a \emph{hyperharmonic series}. Further, we propose another variant of our algorithm that scales to high dimensions. We show theoretically that for both our algorithms, the cumulative regret grows at sub-linear rates. Our experiments with synthetic and real-world optimisation tasks demonstrate the superiority of our algorithms over the current state-of-the-art methods for Bayesian optimisation in unknown search space.
\end{abstract}
\section{Introduction}
Bayesian optimisation (BO) is a powerful and flexible tool for  efficient global optimisation of expensive black-box functions. An underlying limitation of existing approaches is that the search is restricted to a pre-defined and fixed search space, thus implicitly assuming  that this search space  will contain the global optimum. To set suitable bounds of the search space, prior knowledge is required. When exploring entirely new problems (e.g a new machine learning model), such prior knowledge is often poor, and thus the specification of the search space can be erroneous leading to suboptimal solutions. For example, in many machine learning algorithms we have hyperparameters or parameters that can take values in an unbounded space e.g. $L1/L2$ penalty hyperparameters in elastic-net can take any nonnegative value. Similarly, the weights of a neural network can take any real values. No matter how large a finite search space is set, one cannot be sure if this search space contains the global optimum.

This problem is considered in \citep{shahriari16,NguyenGR0V17,ha2019,Wei20}, and UBO \citep{ha2019} is the first to provide global convergence analysis. However, they consider a weak version of the global convergence, i.e., instead of seeking the exact global optimum, they  find a solution wherein the function value is within $\epsilon > 0$ of the global optimum. Further, there is no analysis on the convergence rate of this algorithm which is important for understanding the efficiency of the optimisation.

Another complication arises when high dimensional problems are considered (e.g. hyperparameter tuning \citep{snoek12}, reinforcement learning \citep{CalandraSPD16}), as BO scales poorly in practice. With unknown search spaces, we need to consider evolving/growing search spaces. This growth in search spaces makes high dimensional BO further challenging as the search space is already exponentially large with respect to dimensions. Both these challenges compound the difficulty in maximising the acquisition functions. With limited budgets, the accuracy of points suggested by the acquisition step is often poor and this adversely affects both the convergence and the efficiency of the BO algorithm. Thus solutions to  BO in unknown high-dimensional search spaces need to be found.

In this paper, we address these open problems. Our contributions are as follows:
\begin{itemize}
  \item We introduce a novel BO algorithm for unknown search space, using a volume expansion strategy with a rate of expansion controlled through \emph{hyperharmornic series} \citep{Chlebus2009}. We show that our algorithm achieves a sub-linear convergence rate.
  \item We then provide a first solution for BO problem with unknown high dimensional search spaces. Our solution is based on using a restricted search space consisting of a set of hypercubes with small sizes. Based on controlling the number of hypercubes according to the expansion rate of the search space, we derive an upper bound on the cumulative regret and theoretically show that it can achieve a sub-linear growth rate.
  \item We evaluate our algorithms extensively using a variety of optimisation tasks including optimisation of several benchmark functions and tuning both the hyperparameters (Elastic Net) and parameters of machine learning algorithms (weights of a neural network and Lunar Lander). We demonstrate that our algorithms have better sample and computational efficiency compared to existing methods on both synthetic and real optimisation tasks. Our source code is publicly available at $\texttt{https://github.com/Tran-TheHung/Unbounded\_Bayesian\_Optimisation}$.
\end{itemize}
\section{Related Work}
There are two main approaches in previous work addressing BO with unknown search spaces. The first tackles the problem by nullifying the need to declare the search space - instead a regularized acquisition function is optimised on an unbounded search space such that its maximum can never be at infinity. However, this approach requires critical parameters that are difficult to specify in practice, and there is no theoretical guarantee on the optimisation efficiency. The second approach uses volume expansion - starting from a user-defined region, the search space is sequentially expanded during optimisation. The simplest strategy repeatedly doubles the volume of the search space every few iterations \citep{shahriari16}. Such a strategy is not efficient as it grows the search space exponentially. \citep{Nguyen19} propose a expansion strategy  based on a filter, however they require an additional crucial assumption that the initial search space is sufficiently close to the optimum. Further, their regret bound is non vanishing. More recently, \citep{Wei20} propose an adaptive expansion strategy based on the uncertainty of the GP model, but do not provide convergence guarantees. A recent approach by \citep{ha2019} is the first to provide a global convergence analysis. However, their work has two limitations. First, the convergence analysis aims at $\epsilon$-regret, meaning the algorithm only converges approximately. Second, there is no analysis of the convergence rate. Compared to these works, our approach is novel and is only one to guarantee the sub-linear convergence rate.

In another context, the high dimensional BO has been studied extensively in the literature. In order to make BO scalable to high dimensions, most of the methods make restrictive structural assumptions such as the function having an effective low-dimensional subspace \citep{Wang13,Djolonga13,Garnett14,ErikssonDLBW18,Zhang19,nayebi19a}, or being decomposable in subsets of dimensions \citep{Kandasamy15,li16,rolland18a,Mutny18,HoangHOL18}. Through these assumptions, the acquisition function becomes easier to optimise
and the global optimum can be found. However, such assumptions are rather strong. Without these assumptions, high-dimensional BO
problem is more challenging. There have been a limited attempts to develop scalable BO methods \citep{OhGW18,Johannes19,ErikssonPGTP19,Tran-The0RV20}. To our knowledge, all these works have not been considered in the case of unknown search spaces. We provide the first solution for the unknown high-dimensional problem.  Our solution does not make structural assumptions on the function.
\section{Preliminaries}
Bayesian optimisation (BO) finds the global optimum of an unknown, expensive, possibly non-convex function $f(x)$. It is assumed that we can interact with $f$ only by querying at some $\mathbf{x} \in \mathbb{R}^d$ and obtain a noisy observation $y = f(\mathbf{x}) + \epsilon$ where $\epsilon\sim\mathcal{N}(0,\sigma^2)$. The search space is required to specified a priori and is assumed to include the true global optimum. BO proceeds sequentially in an iterative fashion. At each iteration, a surrogate model is used to probabilistically model $f(\mathbf{x})$. Gaussian process (GP) \citep{Rasmussen05} is a popular choice for the surrogate model as it offers a prior over a large class of functions and its posterior and predictive distributions are tractable. Formally, we have $f(\mathbf{x}) \sim \mathcal {GP}(m(\mathbf{x}), k(\mathbf{x}, \mathbf{x}'))$ where $m(\mathbf{x})$ and $k(\mathbf{x},\mathbf{x}')$ are the mean and the covariance (or kernel) functions. Popular covariance functions include Squared Exponential (SE) kernels, Mat\'ern kernels etc. Given a set of observations $\mathcal D_{1:t} = \{\mathbf{x}_i, y_i\}_{i=1}^t$, the predictive distribution can be derived as $P(f_{t+1}| \mathcal D_{1:t}, \mathbf{x}) = \mathcal N(\mu_{t+1}(\mathbf{x}), \sigma_{t+1}^2(\mathbf{x}))$, where $\mu_{t+1}(\mathbf{x}) = \textbf{k}^T[\mathbf{K} + \sigma^2\textbf{I}]^{-1}\textbf{y} + m(\mathbf{x})$ and $\sigma_{t+1}^2(\mathbf{x}) = k(\mathbf{x}, \mathbf{x}) - \text{\textbf{k}}^{T}[\textbf{K} + \sigma^2\textbf{I}]^{-1}\textbf{k}$. In the above expression we define $\textbf{k }= [k(\mathbf{x}, \mathbf{x}_1), ..., k(\mathbf{x}, \mathbf{x}_t)]$, $\textbf{K }= [k(\mathbf{x}_i, \mathbf{x}_j)]_{1 \le i,j \le t}$ and $\textbf{y}=[y_1,\ldots,y_t]$.

After the modeling step, an acquisition function is used to suggest the next $\mathbf{x}_{t+1}$ where the function should be evaluated. The acquisition step uses the predictive mean and the predictive variance from the surrogate model to balance the exploration of the search space and exploitation of current promising regions. Some examples of acquisition functions include Expected Improvement (EI) \citep{Mockus74}, GP-UCB \citep{Srinivas12} and PES \citep{LobatoHG14}. We use GP-UCB acquisition function which is defined  as
\begin{eqnarray}
u_{t}(\mathbf{x}) =\mu_{t-1}(\mathbf{x}) + \sqrt{\beta_{t}}\sigma_{t-1}(\mathbf{x}),
\label{eq:1}
\end{eqnarray}
where $\beta_{t}$  balances the exploration and the exploitation (see \citep{Srinivas12}).
\paragraph{Cumulative Regret:}To measure the performance of a BO algorithm, we use the regret, which is the loss incurred by evaluating the function at $x_t$, instead of at unknown optimal input, formally $r_t = f(\mathbf{x}^*) - f(\mathbf{x}_t)$.  The cumulative regret is defined as $R_T = \sum_{1 \le t \le T}r_t$ , the sum of regrets incurred over given a horizon of $T$ iterations. If we can show that $\lim_{T \rightarrow \infty}\frac{R_T}{T} = 0$, the cumulative regret is \textbf{sub-linear}, and so the algorithm efficiently converges to the optimum.
\section{Problem Setup and HuBO Algorithm}
Bayesian optimisation aims to find the global optimum of black-box functions, i.e.
$$\mathbf{x}^* = \text{argmax}_{\mathbf{x} \in \mathbb{R}^d}f(\mathbf{x}),$$
where $d$ is the input dimension of the search space. Differing from traditional BO where the search space is assumed to be known a priori, we assume that the search space is \textbf{unknown}. As in \citep{ha2019}, we  assume that $\mathbf{x}^*$ is not at infinity to make the BO tractable.

When the search space is unknown, one heuristic solution is to specify it arbitrarily. However, there are two problems: (1) an arbitrary search space that is finite, no matter how large, may not contain the global optimum (2) optimisation efficiency decreases with increasing size of the search space. 

We propose a volume expansion strategy such that the search space can eventually cover the whole $\mathbb{R}^d$ (therefore, guaranteed to contain unknown $\mathbf{x}^*$), while the expansion rate is kept slow enough so that the algorithm efficiently converges, i.e.
$\lim_{T \rightarrow \infty} \frac{ \sum_{1 \le t\le T}(f(\mathbf{x}^*) - f(\mathbf{x}_t))}{T} \rightarrow 0$, given any $T > 0$. To do this, our key idea is to iteratively expand and shift the search space toward the "promising regions". At iteration $t$, we expand the search space by $\mathcal{O}(t^{\alpha})$, where $\alpha < 0$. We choose this form so that the search space expansion slows over time. The parameter $\alpha$ is set to guarantee the efficient convergence. Our volume expansion strategy is as follows: starting from an initial user-defined region, denoted by $\mathcal X_0= [a, b]^d$, the search space at iteration $t$, denoted by $\mathcal X_t = [a_t, b_t]^d$ will be built from $\mathcal X_{t-1} = [a_{t-1}, b_{t-1}]^d$ by a sequence of transformations as follows:
\begin{eqnarray}
\mathcal X_{t-1} \rightarrow \mathcal X'_t \rightarrow \mathcal X_t
\label{transformation}
\end{eqnarray}
where $\mathcal X'_t = [a'_t, b'_t]^d$ is expanded from $\mathcal X_{t-1}$ by the $(\frac{b-a}{2}) t^{\alpha}$ increment in each direction for all dimensions as $a'_t = a_{t-1} - \frac{b-a}{2}t^{\alpha}$ and $b'_t = b_{t-1} + \frac{b-a}{2}t^{\alpha}$.

To build $\mathcal X_t$ from $\mathcal X'_t$, we translate the center of $\mathcal X'_t$, denoted by $\mathbf{c}'_t$ toward the best solution found until iteration $t$. To avoid a "fast" translation of $\mathbf{c}'_t$, which could cause the divergence, we use a fixed, finite domain  $\mathcal C_{initial}$ to restrict the translation of $\mathbf{c}'_t$. We translate $\mathbf{c}'_t$ toward a point $\mathbf{c}_t$ where $\mathbf{c}_t \in \mathcal C_{initial}$ is the closest point to the best solution found until iteration $t$. In practice, our algorithm would typically benefit by setting a large  $\mathcal C_{initial}$ as this allows the search space in iteration $t$ to be centred close to the best found solution. However, irrespective of the size of $\mathcal C_{initial}$, as we show in our convergence analysis, our search space expansion scheme is still guaranteed to converge to $\mathbf{x}^*$.

In this transformation, step $\mathcal X_{t-1} \rightarrow \mathcal X'_t$ plays the role to expand the search space. Step $X'_t \rightarrow \mathcal X_t$ plays the role to translate the search space towards the promising region surrounding the best solution found so far. By induction, we can compute the volume of $\mathcal X_t$ as $Vol(\mathcal X_t)$: $Vol(\mathcal X_t) = (b-a)^d(1 + \sum_{j=1}^t j^{\alpha})^d$. Therefore, given any $t$, the volume of the search space $\mathcal X_t$ is controlled by a partial sum of a \emph{hyperharmonic series} $\sum_{j=1}^{t} j^{\alpha}$ \citep{Chlebus2009}.

Our strategy called the $\text{\textbf{H}yperharmonic \textbf{u}nbounded \textbf{B}ayesian \textbf{O}ptimisation}$(HuBO) is described in Algorithm \ref{alg:alg1}. It closely follows the standard BO algorithm. The only difference lies in the acquisition step where instead of using a fixed search space, the search space expands in each iteration following (\ref{transformation}). We use the GP-UCB acquisition function with $\beta_t = 2log(4\pi_t/\delta) + 4dlog(dts_2(b-a)(1 + \sum_{j=1}^{t} j^{\alpha})\sqrt{log(4ds_1/\delta)})$, where $\sum_{t\ge 1} \pi_t^{-1} =1$, $\pi_t > 0$.
\begin{algorithm}[tb]
\caption{$\text{HuBO}$ Algorithm}
\label{alg:alg1}
\textbf{Parameters}: $\alpha \in \mathbb{R}$- rate of expanding the volume of the search space\\
\textbf{Initialisation}: Define an initial search space $\mathcal X_0= [a, b]^d$, a finite domain $\mathcal C_{initial} =[c_{min}, c_{max}]^d$, where $\mathcal X_0 \subseteq \mathcal C_{initial}$. Sample initial points in $\mathcal X_0$ to build $\mathcal D_{0}$.
\begin{algorithmic}[1]
\FOR{$t = 1, 2, ...T $}
    \STATE Fit the Gaussian process using $\mathcal D_{t-1}$.
    \STATE Define $\mathcal X_t$ using (\ref{transformation}).
    \STATE Find $\mathbf{x}_t = \text{argmax}_{\mathbf{x} \in \mathcal X_t} u_t(\mathbf{x})$, where $u_t(\mathbf{x})$ defined as in Eq (\ref{eq:1}) to find $\mathbf{x}_t$.
    \STATE Sample $y_t = f(\mathbf{x}_t) + \epsilon_t$.
    \STATE Augment the data $\mathcal D_{t} = \{\mathcal D_{t-1}, (\mathbf{x}_{t}, y_t)\}$.
\ENDFOR
\end{algorithmic}
\end{algorithm}
\vspace*{-2mm}
\subsection{Convergence Analysis of HuBO Algorithm}
In this section, we provide the convergence analysis of proposed HuBO Algorithm. \textbf{All proofs are provided in the Supplementary Material}. To guarantee the convergence, the first necessary condition is that the search space eventually contains $\mathbf{x}^*$.
\begin{theorem}[Reachability]
If $\alpha \ge -1$, then the HuBO algorithm guarantees that there exists a constant $T_0 > 0$ (independent of $t$) such that when $t > T_0$,  $\mathcal X_t$ contains $\mathbf{x}^*$.
\label{theorem:11}
\end{theorem}
\vspace*{-2mm}
\begin{proof}
We denote the center of the user-defined finite region $\mathcal C_{initial} = [c_{min},c_{max}]^d$ as $\mathbf{c}_0$. By the assumption of $\mathbf{x}^*$  being not at infinity, there exists a smallest range $[a_g, b_g]^d$ so that both $\mathbf{x}^*$ and  $\mathbf{c}_0$ belong to  $[a_g, b_g]^d$. By induction, the search space $\mathcal X_t$ at iteration $t$ is a hypercube, denoted by $[a_t, b_t]^d$. Following our search space expansion, the center of $\mathcal X_t$ only moves in region $\mathcal C_{initial}$. Therefore, for each dimension $i$,  we have in the worst case,  $(b_t - [\mathbf{c}_0]_i)$ is at least $\frac{c_{min}-c_{max}}{2} + \frac{b_t -a_t}{2}$ and $([\mathbf{c}_0]_i - a_t)$ is at most $\frac{c_{max} -c_{min}}{2} - \frac{b_t - a_t}{2}$. By induction, we can compute the length of $\mathcal X_t$ as $b_t - a_t: b_t- a_t = (b-a)(1+\sum_{j=1}^{t} j^{\alpha})$.
Therefore, $(b_t - [\mathbf{c}_0]_i)$ is at least $ \frac{c_{min}-c_{max}}{2} + \frac{b-a}{2}(1+\sum_{j=1}^{t} j^{\alpha})$ and $([\mathbf{c}_0]_i - a_t)$ is at most $\frac{c_{max}-c_{min}}{2} - \frac{b-a}{2}(1+\sum_{j=1}^{t} j^{\alpha})$.

If there exists a $T_0$ such that two conditions satisfy: (1) $\frac{c_{min}-c_{max}}{2} + \frac{b-a}{2}(1+\sum_{j=1}^{T_0} j^{\alpha}) \ge b_g$, and (2) $\frac{c_{max}-c_{min}}{2} - \frac{b-a}{2}(1+\sum_{j=1}^{T_0} j^{\alpha})  < a_g$, then we can guarantee that for all $t>T_0$, the search space $\mathcal X_t$ will contain $[a_g, b_g]^d$ and thus also contain $\mathbf{x}^*$. Such a $T_0$ exists because $\sum_{j=1}^{t} j^{\alpha}$ is a diverging sum with $t$ when $\alpha \ge -1$ \citep{Chlebus2009}. From the conditions (1) and (2), we can see that $T_0$ is a function of parameters $a$,  $b$, $c_{min}$, $c_{max}$, $\alpha$ and $a_g, b_g$. We provide the complete proof in the Supplementary. \qedhere
\vspace*{-2mm}
\end{proof}
Using the existence of $T_0$, we derive a cumulative regret for our proposed HuBO algorithm by applying the techniques of GP-UCB
as in \citep{Srinivas12}, however, with an adaptation according to the growth of search spaces over time.
\begin{theorem}[Cumulative Regret $R_T$ of HuBO Algorithm]
Let $f \sim \mathcal{GP}(\mathbf{0}, k)$ with a stationary covariance function $k$. Assume that there exist constants $s_1, s_2 > 0$ such that $\mathbb{P}[sup_{\mathbf{x} \in \mathcal X}|\partial f/ \partial x_{i}| > L] \le s_1e^{-(L/s_2)^2}$
for all $L > 0$ and for all $i \in \{1,2,..., d\}$. Pick a $\delta \in (0,1)$. Thus, if  $-1 \le \alpha < 0$ then for any horizon $T > T_0$, the cumulative regret of the proposed HuBO algorithm is bounded as
\begin{itemize}
  \item \scalebox{0.9}{$R_T \le \mathcal O^*(T^{\frac{(\alpha +1)d +1}{2}})$} if $k$ is a SE kernel,
  \item \scalebox{0.9}{$R_T \le \mathcal O^*(T^{\frac{d^2(\alpha +2) + d}{4\nu + 2d(d+1)}})$} if $k$ is a Mat\'ern kernel
\end{itemize}
with probability greater than $1 -\delta$.
\label{cumulative_regret_11}
\end{theorem}
\paragraph{Sub-linear Regret}
By Theorem \ref{cumulative_regret_11}, the HuBO algorithm obtains a sub-linear cumulative regret for SE kernels if $-1 \le \alpha < -1 + \frac{1}{d}$, and for Mat\'ern kernels if $-1 \le \alpha < \text{min}\{ 0, -1 + \frac{2\nu}{d^2}\}$.
\section{HD-HuBO Algorithm in High Dimensions}
Further, we extend the HuBO algorithm for high dimensional spaces. As discussed in the introduction, maximisation of the acquisition function in a crucial step when working with unknown high dimensional search spaces. Given the same computation budget, the larger the expanded search space, the less accurate the maximiser suggested by the acquisition step is. To improve this step, our solution is to restrict the search space. We propose a novel volume expansion strategy as follows. Starting from $\mathcal X_0$, the search space $\mathcal H_t$ at iteration $t$ with $t \ge 1$ is defined as
\begin{eqnarray}
\mathcal H_t = \{H(\mathbf{z}^{1}_t, l_h) \cup ...\cup H(\mathbf{z}^{N_t}_t, l_h)\} \cap \mathcal X_t,
\end{eqnarray}
where $\mathcal X_t$ is the search space of HuBO and is defined in \ref{transformation}, $H(\mathbf{z}^{i}_t, l_h)$ is a $d$-dimensional hypercube centered at $\mathbf{z}^{i}_t$ with size $l_h$, and $N_t$ denotes the number of such hypercubes at iteration $t$. To handle the computational requirement, we choose $l_h$ to be small. Thus, at an iteration $t$, we maximise the acquisition function on only this finite set of hypercubes in $\mathcal X_t$ with small size.

Importantly, we can show that the maximisation on such hypercubes can result in low regret by proposing
a strategy to choose the set of hypercubes. Formally, at iteration $t$ we choose $N_t = N_0\ceil*{t^{\lambda}}$ where $\lambda \ge 0$, $N_0 \in \mathbb{N}$ and $N_0 \ge 1$. We choose $N_t$ hypercubes with centres $\{\mathbf{z}^{i}_t\}$ which are sampled uniformly at random from $\mathcal X_t$. We refer to this algorithm as HD-HuBO which is described in Algorithm \ref{alg:alg2}. We use the acquisition function $u_t(\mathbf{x})$ with $\beta_t = 2log(\pi^2t^2/\delta) +  2dlog(2s_2l_hd\sqrt{log(6ds_1/\delta)}t^2)$, where $s_1, s_2$ is defined in Theorem \ref{theorem44}.
\begin{algorithm}[tb]
\caption{HD-HuBO Algorithm}
\label{alg:alg2}
\textbf{Parameters}: $\alpha \in \mathbb{R}$- rate of expanding the search space, $\lambda \in \mathbb{R^+}$ and $N_0$- the parameters related to the number of hypercubes, $l_h \in \mathbb{R^{+}}$- the size of hypercubes\\
\textbf{Initialisation}: Define an initial space $\mathcal X_0= [a, b]^d$, an initial domain $\mathcal C_{initial} =[c_{min}, c_{max}]^d$, where $\mathcal X_0 \subseteq \mathcal C_{initial}$. Sample initial points in $\mathcal X_0$ to construct $\mathcal D_{0}$.
\begin{algorithmic}[1] 
\FOR{$t = 1, 2, ...T $}
    \STATE Fit a Gaussian process using $\mathcal D_{t-1}$.
    \STATE Update the search space $\mathcal H_t = \{H(\mathbf{z}^1_t, l_h) \cup ...\cup H(\mathbf{z}^{N_t}_t, l_h)\} \cap \mathcal X_t$, where $N_t = N_0\ceil*{t^{\lambda}}$ and $N_t$ values of $\mathbf{z}^{i}_t$ are drawn uniformly at random from $\mathcal X_t$.
    \STATE Find $\mathbf{x}_t = \text{argmax}_{\mathbf{x} \in \mathcal H_t} u_t(\mathbf{x})$
    \STATE Sample $y_t = f(\mathbf{x}_t) + \epsilon_t$.
    \STATE Augment the data $\mathcal D_{t} = \{\mathcal D_{t-1}, (\mathbf{x}_{t}, y_t)\}$
\ENDFOR
\end{algorithmic}
\end{algorithm}
\vspace*{-4mm}
\subsection{Convergence Analysis for HD-HuBO Algorithm}
In this section, we analyse the convergence of our proposed HD-HuBO algorithm. Similar to HuBO algorithm, a Reachability property is necessary to guarantee the convergence. On the restricted search space $\mathcal H_t$, it is a crucial challenge. To overcome this, we estimate the distance between $\mathbf{x}^*$ and $\mathbf{x}^*_t$ which is the closest point to $\mathbf{x}^*$ in the search space $\mathcal H_t$, as shown in the following Theorem \ref{reachbility22}. Thus, although we cannot maintain the Reachability property as in HuBO algorithm, we can still obtain a similar Reachability property with high probability if $\lambda > d(\alpha +1)$ and $-1 \le \alpha < 0$, where $\alpha$ is the expansion rate of the search space and is defined as in HuBO algorithm.
\begin{theorem}
Pick a $\delta \in (0,1)$. Let $\mathbf{x}^*_t \in \mathcal H_t$ be the closest point to $\mathbf{x}^*$ in the search space $\mathcal H_t$. For any $t > T_0$ and $-1 \le \alpha < 0$, with probability greater than $1- \delta$, we have
\begin{equation}
||\mathbf{x}^*_t - \mathbf{x}^*||_2 < \frac{2(b-a)}{\pi}(\Gamma(\frac{d}{2} +1))^{\frac{1}{d}}(log(\frac{1}{\delta}))^{\frac{1}{d}}M_t,
\end{equation}
where the constant $T_0$ is defined in Theorem \ref{theorem:11}, $\Gamma$ is the gamma function, and $ M_t= (2 + ln(t))t^{-\frac{\lambda}{d}}$ if $\alpha = -1$, otherwise, $ M_t = 2(\alpha +1)^{-1}t^{-\frac{\lambda}{d}}$ if $-1 < \alpha < 0$.
\label{reachbility22}
\end{theorem}
By Theorem \ref{reachbility22}, for both cases $\alpha = -1$ and $-1 < \alpha < 0$, $\lim_{t \rightarrow \infty} M_t \rightarrow 0$ if $\lambda > d(\alpha +1)$. Therefore, $\lim_{t \rightarrow \infty} ||\mathbf{x}^*_t - \mathbf{x}^*||_2 \rightarrow 0$ if $\lambda > d(\alpha +1)$. The Reachability property is guaranteed with high probability.

Using Theorem \ref{reachbility22}, we derive a cumulative regret for our proposed HD-HuBO algorithm as follows.
\begin{theorem}[Cumulative Regret $R_T$ of HD-HuBO Algorithm]
Let $f \sim \mathcal{GP}(\mathbf{0}, k)$ with a stationary covariance function $k$. Assume that there exist constants $s_1, s_2 > 0$ such that $\mathbb{P}[sup_{\mathbf{x} \in \mathcal X}|\partial f/ \partial x_{i}| > L] \le s_1e^{-(L/s_2)^2}$
for all $L > 0$ and for all $i \in \{1,2,..., d\}$. Pick a $\delta \in (0,1)$. Then, with $T > T_0$, under conditions $\lambda > d(\alpha +1)$, $-1 \le \alpha < 0$, $l_h > 0$, the cumulative regret of proposed HD-HuBO algorithm is bounded as
\begin{itemize}
  \item \scalebox{0.9}{$R_T \le \mathcal O^*(T^{\frac{(\alpha +1)d +1}{2}} + (log(\frac{6}{\delta}))^{\frac{1}{d}}B_T)$} if $k$ is a SE kernel,
  \item \scalebox{0.9}{$R_T \le \mathcal O^*(T^{\frac{d^2(\alpha +2) + d}{4\nu + 2d(d+1)}+ \frac{1}{2}} + (log(\frac{6}{\delta}))^{\frac{1}{d}}B_T)$} if $k$ is a Mat\'ern kernel,
\end{itemize}
with probability greater than $1 -\delta$, where $B_T =U_T V_T$ such that $U_T = 2 + ln(T)$ if $\alpha =-1$, otherwise $U_T = 2(\alpha +1)^{-1}$, and $V_T = 1 + ln(T)$ if $\lambda = d$, otherwise $V_T = 1 + \frac{d}{d- \lambda} \text{max}\{ 1, T^{1 -\frac{\lambda}{d}}\}$.
\label{theorem44}
\end{theorem}
\vspace*{-5mm}
\paragraph{Sub-linear Regret} The upper bound on $R_T$ is sub-linear because we have $\lim_{T \rightarrow \infty} \frac{B_T}{T} = 0$ for all the cases of $U_T$ and $V_T$. The conditions on $\alpha$ is maintained as in HuBO to guarantee the sub-linear regret for SE kernels and Mat\'ern kernels. Together with conditions on $\lambda$, our proposed HD-HuBO obtains a sub-linear cumulative regret for SE kernels if $-1 \le \alpha < -1 + \frac{1}{d}$ and $\lambda > d(\alpha +1)$, and for Mat\'ern kernels if $-1 \le \alpha < -1 + \frac{2\nu}{d^2}$ and $\lambda > d(\alpha +1)$. We note that the regret bound of HD-HuBO is higher than HuBO's regret bound, however HD-HuBO uses only the restricted search space of HuBO.
\vspace*{-4mm}
\section{Discussion}
\paragraph{On the use of hypercubes} While Eriksson et al. \citep{ErikssonPGTP19} proposed to use hypercubes for high dimensional BO, our main contribution is a high dimensional BO for \emph{unknown search spaces} (a novel problem setting). Unlike in \citep{ErikssonPGTP19} where the number of hypercubes are fixed, we provide a rigorous method to increase the number of hypercubes with iterations which is required for our case as the search space is unknown and an initial randomly specified search space needs to keep growing to ensure the convergence. Further, in \citep{ErikssonPGTP19}, there is no theoretical analysis of regret, nor there is any rigorous analysis on the number of hypercubes.
\vspace*{-3mm}
\paragraph{On the effect of $\alpha$ and $\lambda$ parameters}
For both our algorithms, to achieve the tightest sub-linear term in the regret, the parameter $\alpha$ needs to be as small as possible while being in the kernel-specific permissible range. However, a small $\alpha$ may lead to a higher value of $T_0$, which may increase finite-time regret. For HD-HuBO, a high $\lambda$ may lead to a larger  volume of the restricted search space. Our algorithm offers a  range of operating choices (through the choice of $\lambda$) while still guaranteeing different grades of sub-linear rate.
\vspace*{-2mm}
\section{Experiments}
To evaluate the performance of our algorithms, HuBO and HD-HuBO, we have conducted a set of experiments involving optimisation of five benchmark functions and three real applications. We compare our algorithms against five baselines: (1) UBO: the method in a recent paper \citep{ha2019}, (2) FBO: the method in \citep{Nguyen19}, (3) Vol2: BO with the search space volume doubled every $3d$ iterations \citep{shahriari16}, (4) Re-H: the Regularized acquisition function with a hinge-quadratic prior \citep{shahriari16}; (5) Re-Q: the Regularized acquisition function with a quadratic prior \citep{shahriari16}.
\begin{figure*}[ht]
\centering
\subfigure{\includegraphics[scale=1.0,width=.32\textwidth,height= .13\textheight]{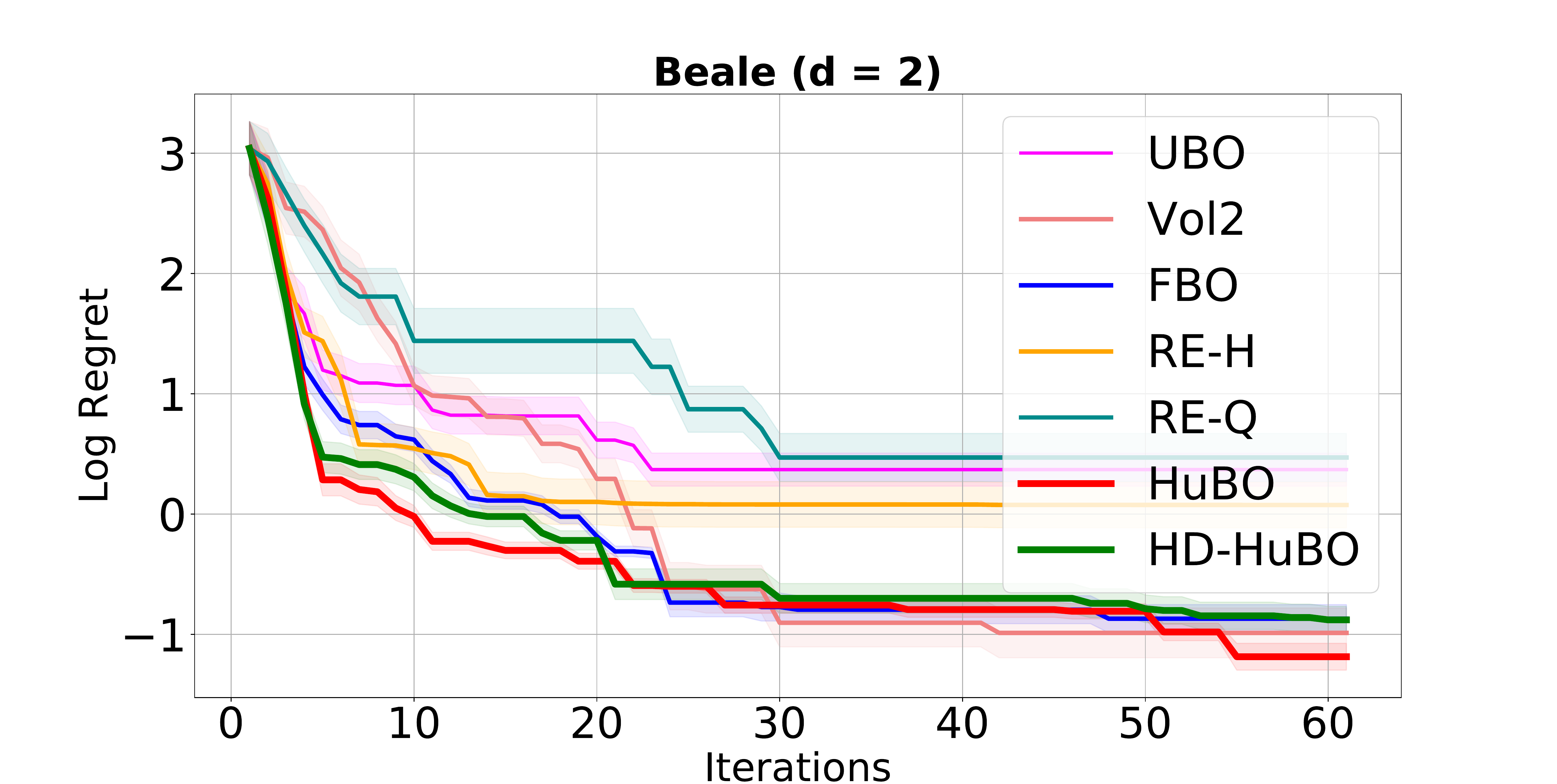}
}\hfill
\subfigure{\includegraphics[scale=1.0,width=.32\textwidth,height= .13\textheight]{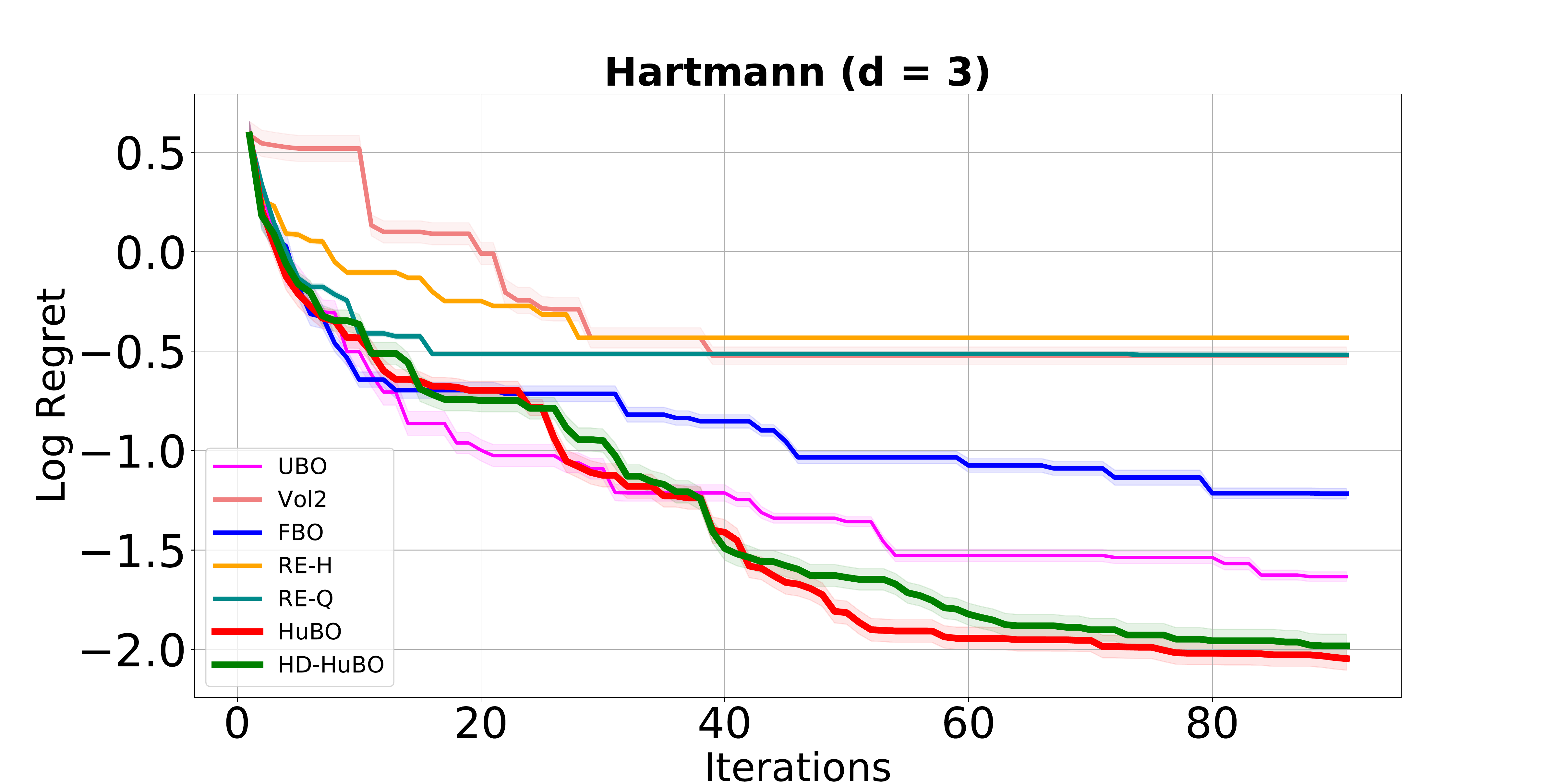}
}\hfill
\subfigure{\includegraphics[scale=1.0,width=.32\textwidth,height= .13\textheight]{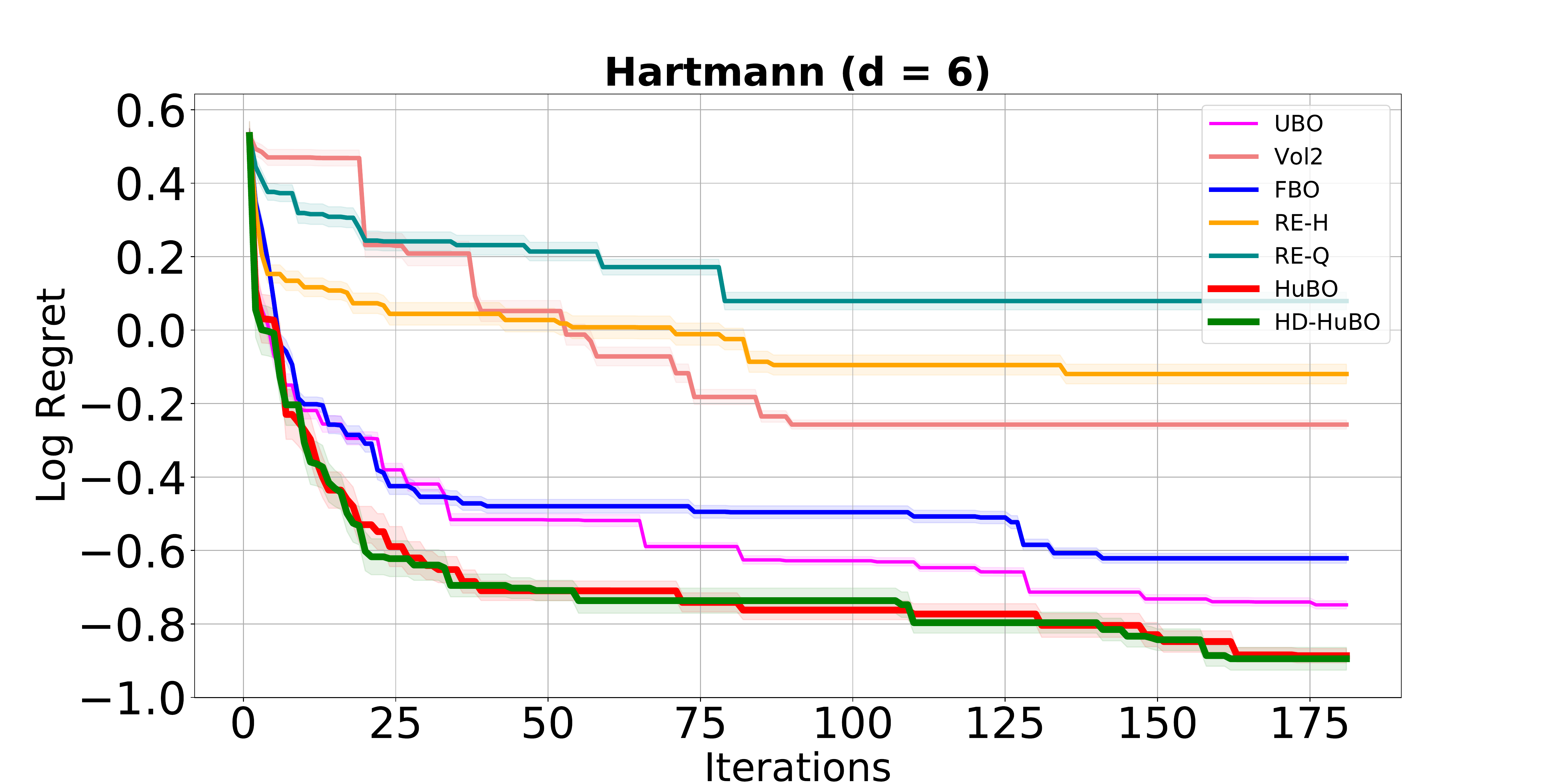}
}
\vspace{-3mm}
\caption{Comparison of baselines and the proposed methods in low dimensions.}
\label{figure1}
\vspace*{-6mm}
\end{figure*}
\begin{figure}[ht]
\centering
\subfigure{\includegraphics[scale=1.0,width=.32\textwidth,height= .13\textheight]{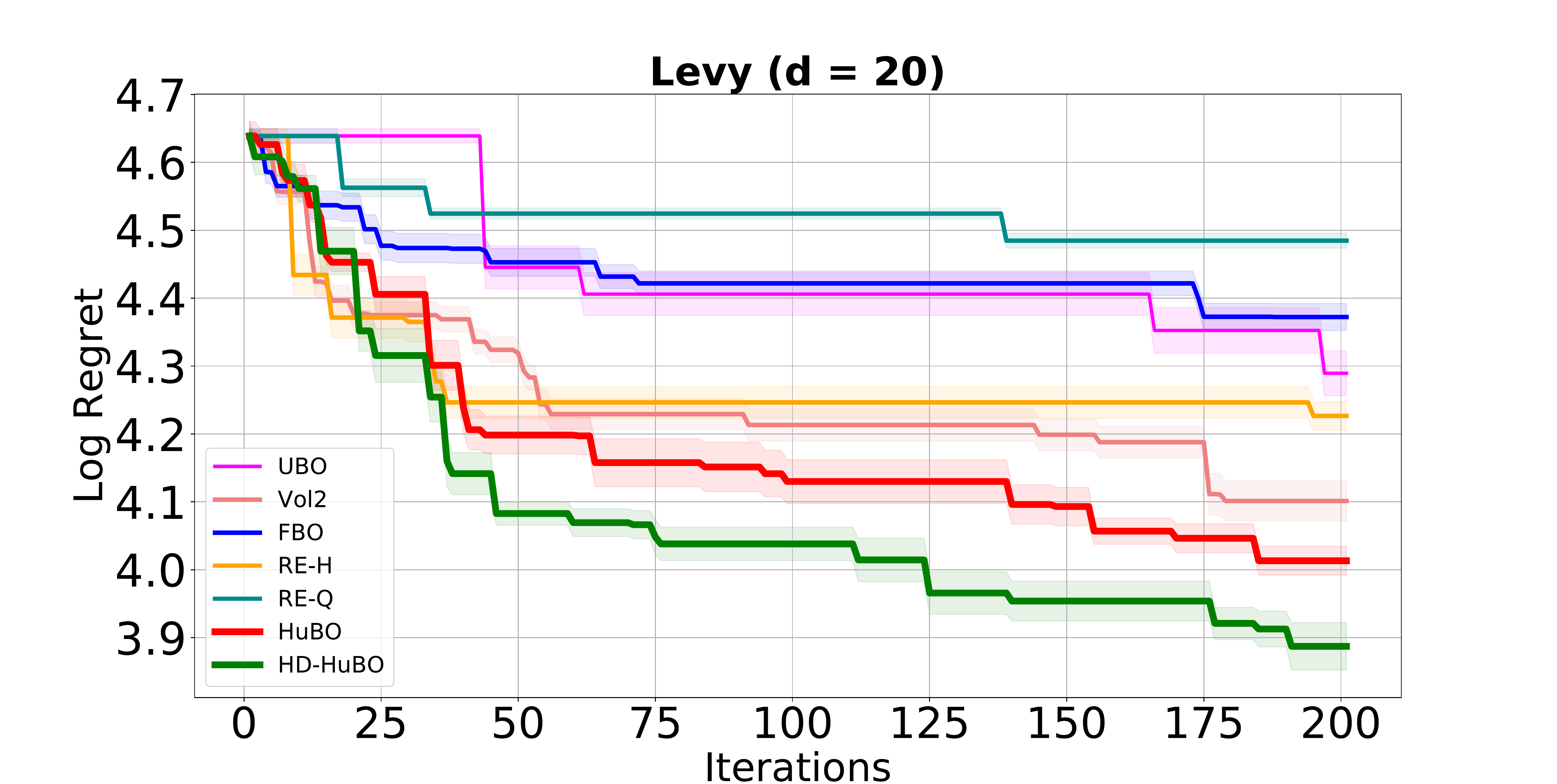}
}\hfill
\subfigure{\includegraphics[scale=1.0,width=.32\textwidth,height= .13\textheight]{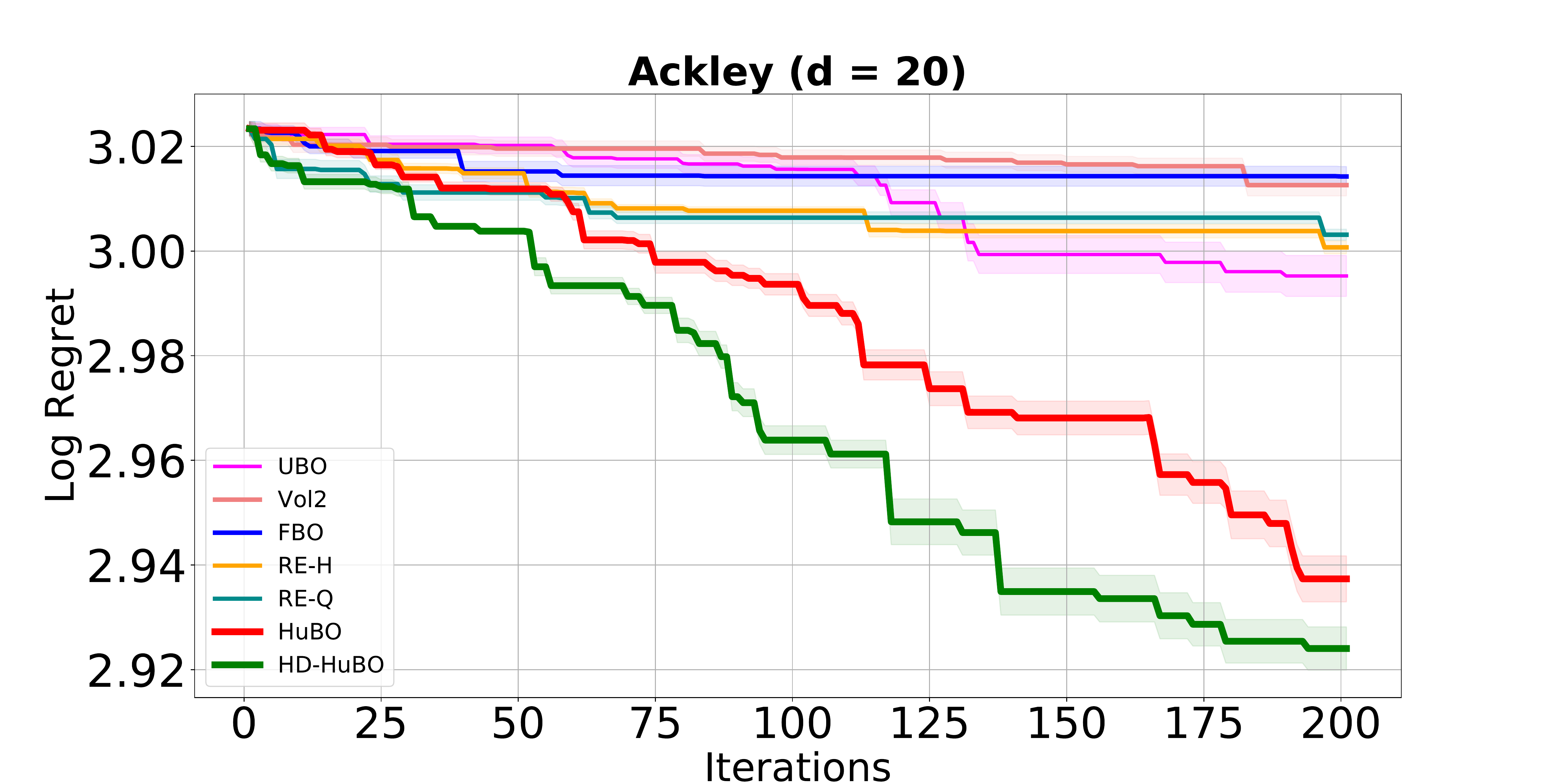}
}\hfill
\subfigure{\includegraphics[scale=1.0,width=.32\textwidth,height= .13\textheight]{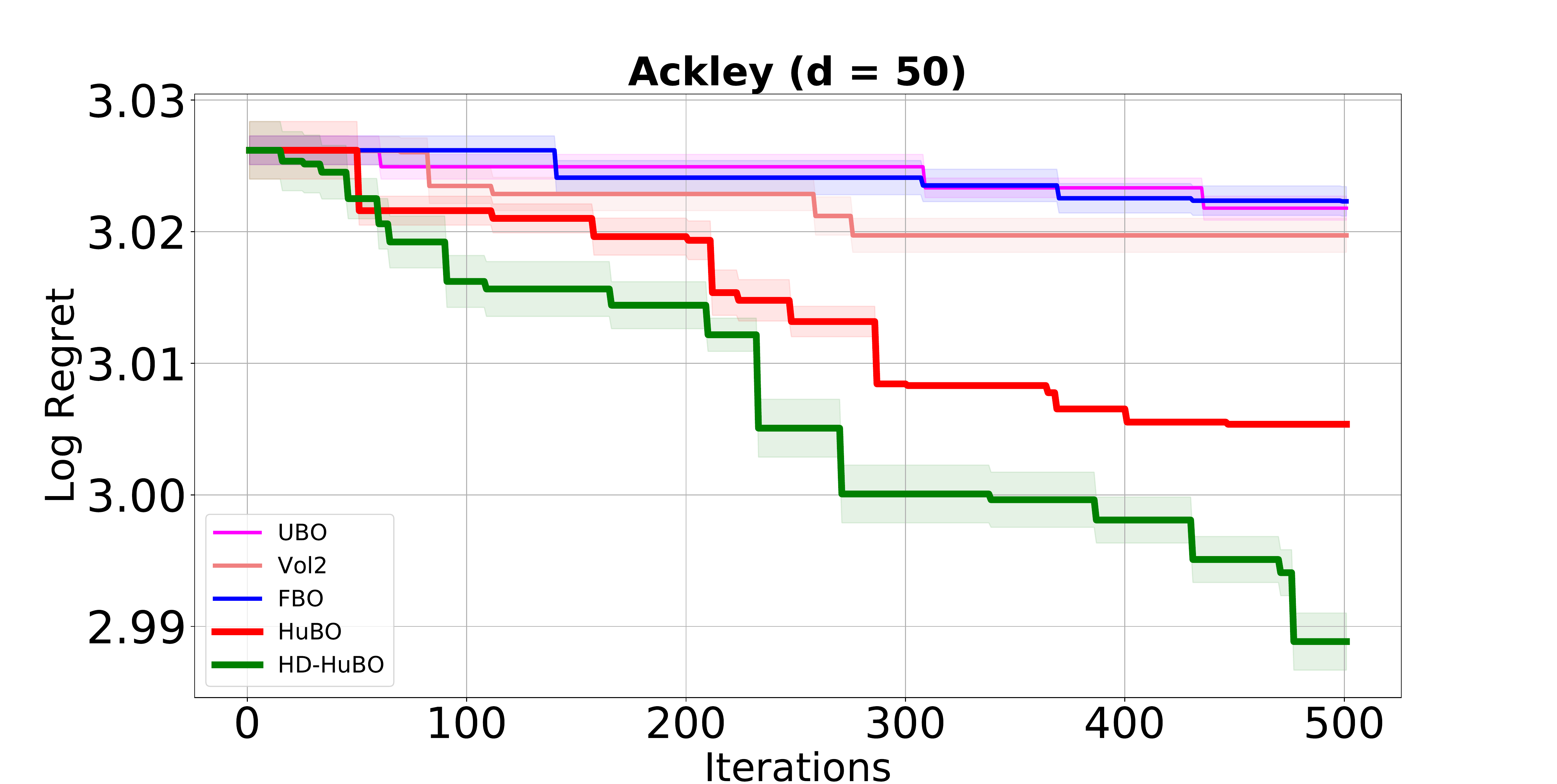}
}
\vspace*{-3mm}
\caption{Comparison of baselines and the proposed methods in high dimensions.}
\label{figure2}
\vspace*{-3mm}
\end{figure}
\paragraph{Experimental settings}
Following the setting of the initial search space $\mathcal X_0$ as in all baselines \citep{shahriari16,Nguyen19,ha2019}, we select the $\mathcal X_0$ as $20\%$ of the pre-defined function domain. For example, if $\mathcal X =[0,1]^d$, the size of $\mathcal X_0$ is the 0.2 where its center is placed randomly in the domain $[0,1]^d$. For our algorithms, we set $\mathcal C_{inital}$ as 10 times to the size of $\mathcal X_0$ along each dimension. We note that we also validate our algorithms by considering additionally a case where $\mathcal X_0$ is only $2\%$ (very small) of the pre-defined function domain. We report this case in the \textbf{supplementary material}.

For all algorithms, the Squared Exponential kernel is used to model GP. The GP models are fitted using the Maximum Likelihood Estimation. The function evaluation budget is set to $30d$ in low dimensions and $10d$ in high dimensions where $d$ is the input dimension. The experiments were repeated 15 times and average performance is reported. For the error bars (or variances), we use the standard error: $\text{Std. Err} = \text{Std. Dev}/\sqrt{n}$, $n$ being the number of runs.
\begin{figure}[ht]
  \centering
  \subfigure{\includegraphics[scale=1.0,width=.40\textwidth]{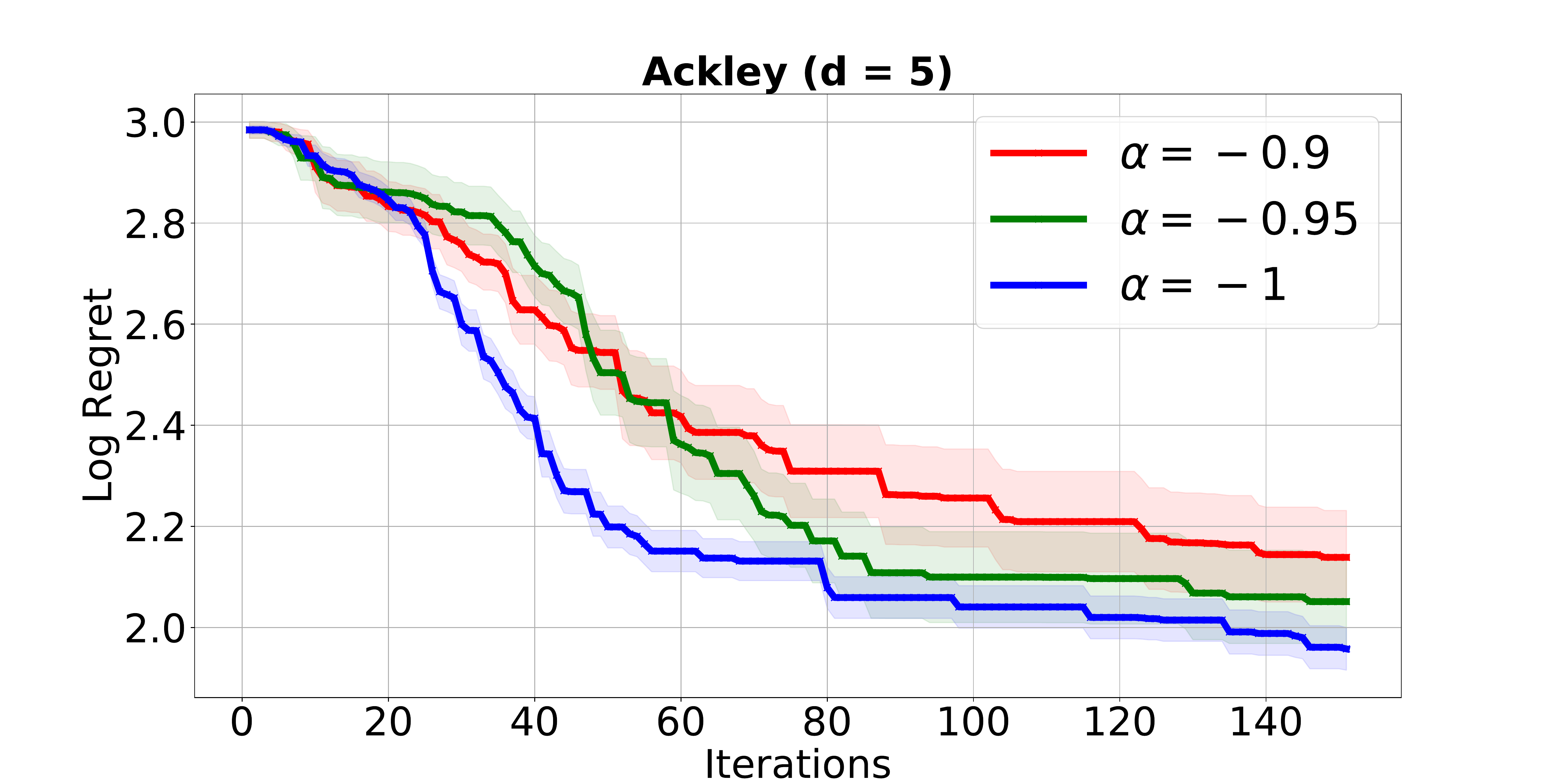}}\quad
  \subfigure{\includegraphics[scale=1.0,width=.40\textwidth]{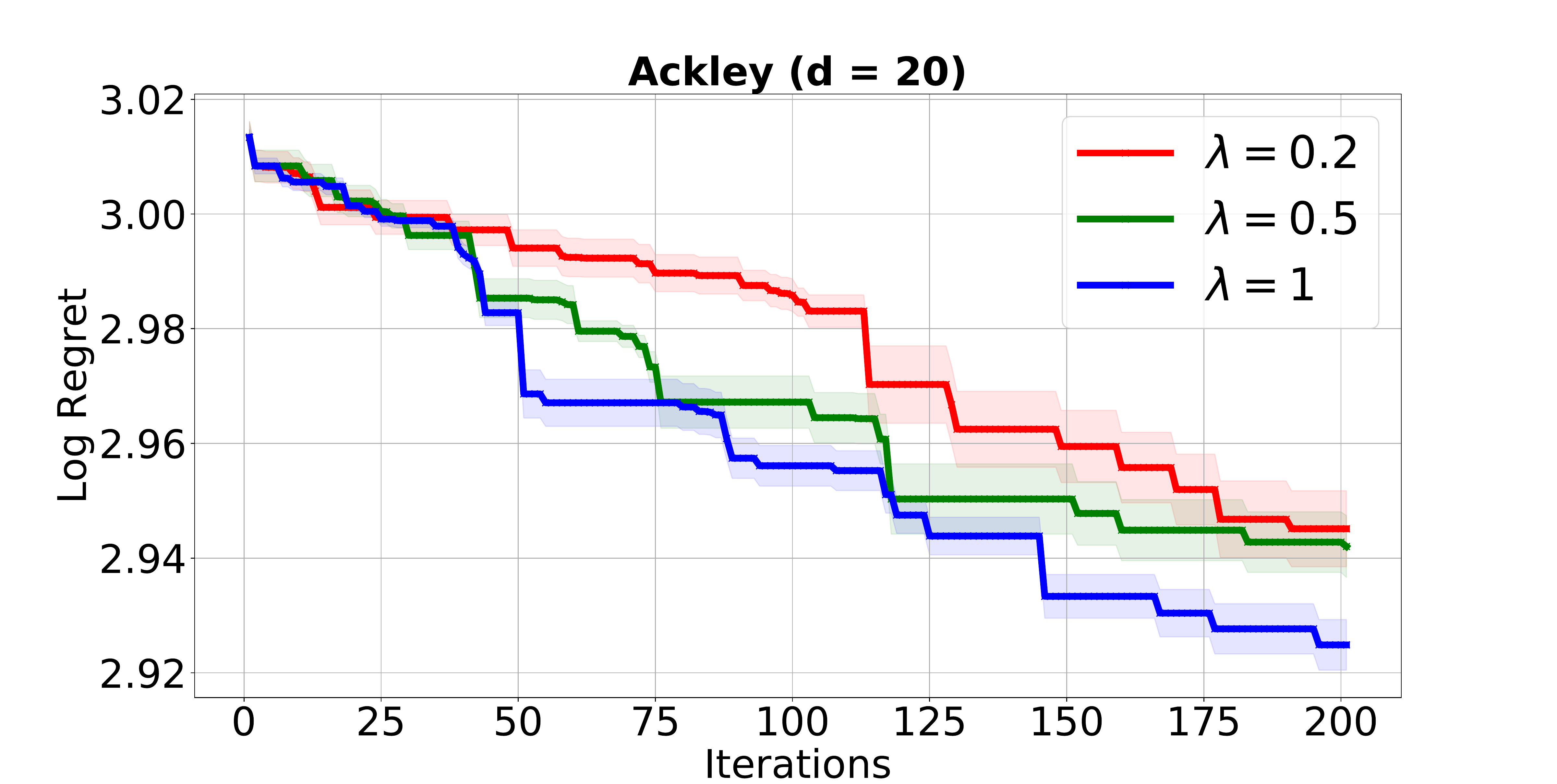}}
\vspace*{-3mm}
\caption{The study of $\alpha$ (for HuBO) and $\lambda$ (for HD-HuBO) in terms of best function value vs iterations. \textbf{Left}: Ackley function (d = 5) and different values of $\alpha$: -0.9; -0.95; -1. \textbf{Right}: Ackley function (d =20) with $\alpha = -1$ and different values of $\lambda$: 0.2; 0.5; 1.}
\vspace*{-2mm}
\label{figure3}
\end{figure}
Following our theoretical results, we choose $\alpha =-1$. By this way, any $\lambda > 0$ is valid as per our theoretical results and more importantly, it allows to minimize the number of hypercubes in HD-HuBO algorithm, and thus reduces the computations. We use $\lambda = 1, N_0 = 1$ and thus $N_t = t$. This means that at iteration $t$, we use $t$ hypercubes for the maximisation of acquisition function. We set the size of hypercubes, $l_h$  as $10\%$ of $\mathcal X_0$. All algorithms are given an equal computational budget to maximise acquisition functions. We also report the average time with each test function in the \textbf{supplementary material} (see Table 1).
\vspace*{-3mm}
\subsection{Optimisation of Benchmark Functions}
We test the algorithms on several benchmark functions: Beale, Hartmann3, Hartmann6, Ackley, Levy functions. We evaluate the progress of each algorithm using the log distance to the true optimum, that is, $\text{log}_{10}(f(\mathbf{x}^*) - f^+)$ where $f^+$ is the best function value found so far. For each test function, we repeat the experiments 15 times. We plot the mean and a confidence bound of one standard deviation across all the runs. Results are reported in Figure \ref{figure1} for low dimensions and in Figure \ref{figure2} for higher dimensions. From Figure \ref{figure1}, we can see that Vol2, Re-H and Re-Q perform poorly in most cases. We note that there is no convergence guarantee for methods Vol2, Re-H and Re-Q. Our HuBO method outperforms baselines.
\begin{figure*}[ht]
\centering
\subfigure{\includegraphics[scale=1.0,width=.32\textwidth, height= .13\textheight]{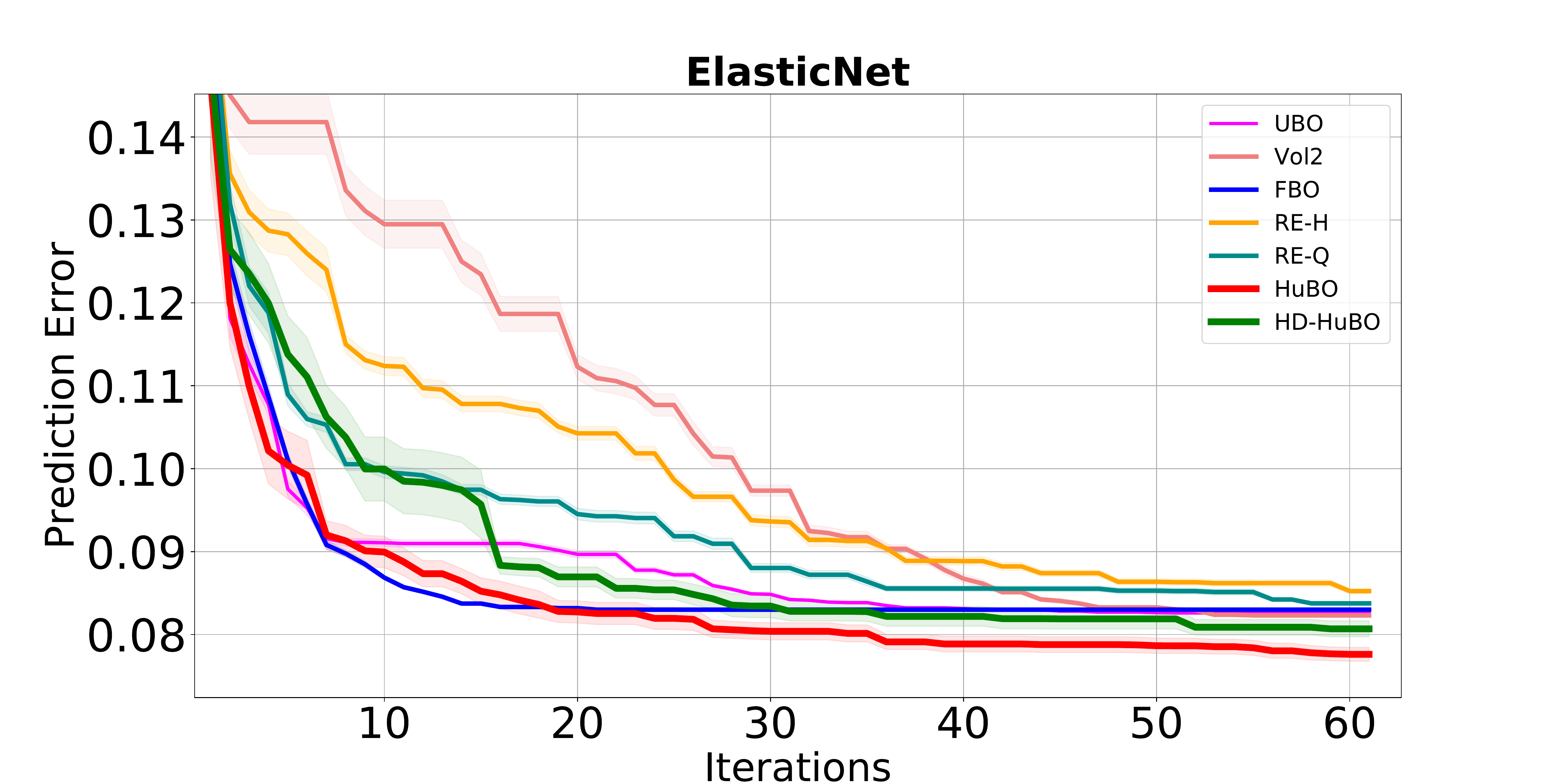}
}\hfill
\subfigure{\includegraphics[scale=1.0,width=.32\textwidth, height= .13\textheight]{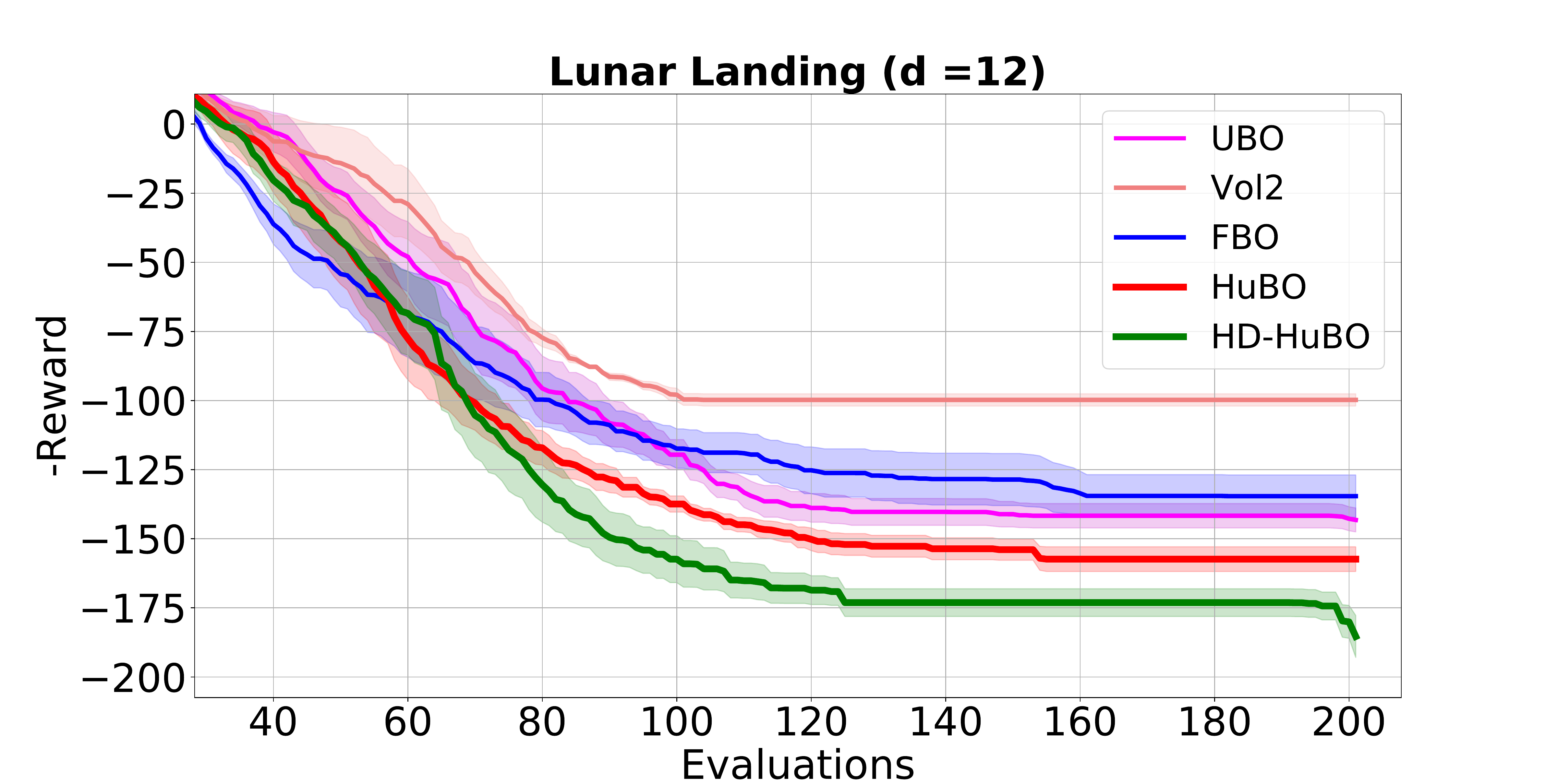}
}\hfill
\subfigure{\includegraphics[scale=1.0,width=.32\textwidth, height= .13\textheight]{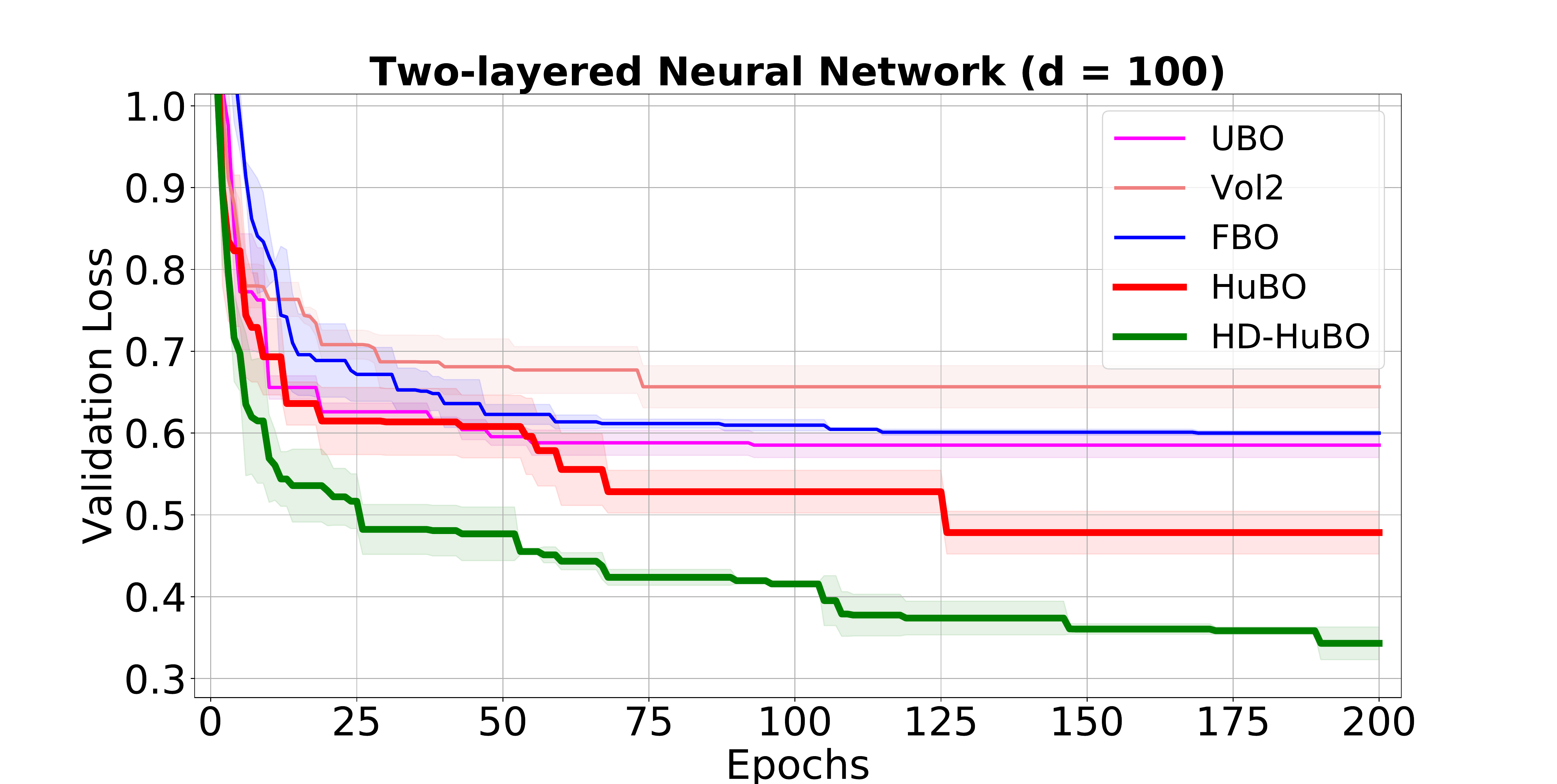}
}
\vspace*{-3mm}
\caption{\textbf{Left}: Prediction accuracy vs Iterations for MNIST dataset using Elastic Net. \textbf{Middle}: Rewards vs Evaluations for 12D Lunar lander. \textbf{Right}: Validation loss vs Epochs for learning parameters of a Two-layered Neural Network.}
\label{figure4}
\vspace*{-3mm}
\end{figure*}
In high dimensions, both HuBO and HD-HuBO outperform baselines. Particularly, since the maximisation of the acquisition function is performed on a restricted search space compared to HuBO, the performance of HD-HuBO is notable.

We would like to emphasise that, different from traditional BO algorithms with fixed search space where error bars (or variances) of
regret curves tend to get tighter over time, in the context of unbounded search space where the search space is being expanded over time,
error bars do not always have this property, they may even become higher over time till the search spaces have not contained the global
optimum. This trend can be seen for many unbounded search space methods such as in \citep{ha2019} and \citep{Nguyen19} in our references.
\paragraph{On the Expansion Rate $\alpha$ and Number of Hypercubes $\lambda$}
The space expansion rate $\alpha$ and the number of hypercubes are control parameters in our method. For SE kernels, $-1 \le \alpha < -1 + \frac{1}{d}$ is needed. However, in high dimensions, this range is tight. Therefore, to test the effect of $\alpha$, we consider HuBO with low dimensions. We create many variants of HuBO ($\alpha = -1, \alpha = -0.95$ and $\alpha = -0.9$). As a testbed, we use 5-dim Ackley function. Figure \ref{figure3} shows that smaller values of $\alpha$ performs better achieving tighter regrets. To study the effect of $\lambda$, we fix $\alpha = -1$ and create three variants of HD-HuBO using $\lambda = 0.2, \lambda = 0.5$ and $\lambda = 1$. We observe that larger values of  $\lambda$ achieve tighter regrets. These results validate our theoretical analysis.
\vspace*{-4mm}
\subsection{Applications to Machine Learning Models}
\paragraph{Elastic Net}
Elastic net is a regression method that has the $L_1$ and $L_2$ regularization parameters. We tune $w_1$ and $w_2$ where $w_1>0$ expresses the magnitude of the regularisation penalty while $w_2 \in [0,1]$ expresses the ratio between the two penalties.
We tune $w_1$ in the normal space while $w_2$ is tuned in an exponent space (base 10). The $\mathcal X_0$ is randomly placed box in the domain $[0,1] \times [-3,-1]$. We implement the Elastic net model by using the function SGDClassifier in the scikit-learn package \citep{elastic}. We train the model using the MNIST train dataset and then evaluate the model using the MNIST test dataset. Bayesian optimisation method suggests a new hyperparameter setting based on the prediction accuracy on the test set. As seen from Figure \ref{figure4} (Left), HuBO performs better than the baselines. In low dimensions, the restriction of the search space in HD-HuBO can influence the efficiency of BO, thus it is less efficient than HuBO.
\vspace*{-3mm}
\paragraph{Lunar Landing Reinforcement Learning}
In this task, the goal is to learn a controller for a lunar lander. The state space for the lunar lander is the position, angle, time derivatives, and whether or not either leg is in contact with the ground. The objective is to maximize the average final reward. The controller we learn is a modification of the original heuristic controller where there are twelve design parameters \citep{ErikssonPGTP19}. Each design parameter is tuned heuristically between $[0,2]$. We follow their setting however instead of assuming a fixed search space as $[0,2]^{12}$, we assume that the search space is unknown. Thus, $\mathcal X_0$ is randomly placed  in the domain $[0,2]^{12}$. We set $\alpha = -1$ and $\lambda = 1$. Our methods HuBO and HD-HuBO eventually learn the best controllers although at early iterations,
\vspace*{-3mm}
\paragraph{Parameter Tuning for Machine Learning Models}
We evaluate the algorithms on a two-layered neural network parameter optimisation task. Here we are given a CNN with one hidden layer of size 10. We denote the weights between the input and the hidden layer by $W_1$ and the weights between the hidden and the output layer by $W_2$. The goal is to find the weights that minimize the loss on the MNIST data set. We optimize $W_2$ by BO methods while $W_1$  are optimized by Adam algorithm. We choose the same network architecture as used by \citep{OhGW18,Tran-The0RV20},  however different from them, we assume that the search space is unknown. We  randomly choose an initial box in the domain $[0, \sqrt{d}]^{100}$. We compare our methods to UBO, Vol2, FBO using the validation loss. We set $\alpha = -1$ and $\lambda = 0.5$ for our methods. Both our methods outperform all the baselines, especially, HD-HuBO.
\vspace*{-2mm}
\section{Conclusion}
\vspace*{-2mm}
We propose a novel BO algorithm for global optimisation in an unknown search space setting. Starting from a randomly initialised search space, the search space shifts and expands as per a hyperharmornic series. The algorithm is shown to efficiently converge to the global optimum. We extend this algorithm to high dimensions, where the search space is restricted on a finite set of small hypercubes so that maximisation of the acquisition function is efficient. Both algorithms are shown to converge with sub-linear regret rates. Application to many optimisation tasks reveals the better sample efficiency of our algorithms compared to the existing methods.

\section*{Broader Impact Statement}
This work has potential to enable the scientists and researchers from the experimental design community to optimise the design of products and processes without the need to specify a search space, which is usually not known accurately when pursuing new products and processes. There are no unethical side of this research or any ill-effects on society.
\section*{Acknowledgments}
This research was partially funded by the Australian Government through the Australian Research Council (ARC).
Prof Venkatesh is the recipient of an ARC Australian Laureate Fellowship (FL170100006).

\bibliographystyle{plainnat}
\bibliography{Hung-research}
\newpage
\appendix
\section*{Supplementary Material}
\def\thesection{\Alph{section}}

In Section A, we first provide some auxiliary results which facilitate the proofs. We present the proofs of Theorem 1, Theorem 2, Theorem 3 and Theorem 4 in next sections. Finally, we provide additional benchmarking results in section F.
\section{Auxiliary Results}
\subsection{Properties of the Volume Expansion Strategy }
\begin{lemma}
For every $t\ge 1$, the search space $\mathcal X_t$ has the $[a_t, b_t]^d$ form where $b_t - a_t = (b-a)(1+\sum_{j=1}^{t} j^{\alpha})$.
\label{space_size}
\end{lemma}
\begin{proof}
We prove the statement by induction. If $t =1$ then by definition of the transformation in section 4, $X'_1 = [a'_1, b'_1]^d$ where $a'_1 = a_0 - \frac{b-a}{2}1^{\alpha}$ and $b'_1 = b_0 + \frac{b-a}{2}1^{\alpha}$. Hence, $b'_1 - a'_1 = 2(b-a)$. By definition of the transformation $X'_1 \rightarrow \mathcal X_1$, the size and the form of $\mathcal X_1$ is preserved from $\mathcal X'_1$. Therefore $\mathcal X_1 = [a_1, b_1]^d$ and $b_1 - a_1 = 2(b-a)$.

We assume that the statement is true for $t \ge 1$. We consider the transformation $\mathcal X_{t} \rightarrow \mathcal X'_{t+1} \rightarrow \mathcal X_{t+1}$. We have $a'_{t+1}= a_{t} - \frac{b-a}{2}(t+1)^{\alpha}$ and $b'_{t+1} = b_{t} + \frac{b-a}{2}(t+1)^{\alpha}$. Hence, $b'_{t+1} - a'_{t+1}= b_t - a_t + (b-a)(t+1)^{\alpha}$. By the inductive hypothesis, we have $\mathcal X_t = [a_t, b_t]^d$ and $b_t - a_t = (b-a)(1+\sum_{j=1}^{t} j^{\alpha})$. Therefore, $b'_{t+1} - a'_{t+1} = (b-a)(1+\sum_{j=1}^{t+1} j^{\alpha})$. By the transformation, the size and the form of $\mathcal X_{t+1}$ is preserved from $\mathcal X'_{t+1}$. Thus, $\mathcal X_{t+1} = [a_{t+1}, b_{t+1}]^d$ where $b_{t+1} - a_{t+1} = (b-a)(1+\sum_{j=1}^{t+1} j^{\alpha})$. The statement holds for any $t \ge 1$.
\end{proof}

Given a finite domain $\mathcal X$, we denote the volume of $\mathcal X$ by $Vol(\mathcal X)$.
\begin{lemma}
For every horizon $T > 0$, set $\mathcal C_T = [c_{min} - (b-a)(1+\sum_{j=1}^{T} j^{\alpha})/2, c_{max} + (b-a)(1+\sum_{j=1}^{T} j^{\alpha})/2]^d$. Then for every $1 \le t \le T$, $X_t \subseteq C_T$.
\label{volume_size}
\end{lemma}
\begin{proof}
We also prove this statement by induction. If $T =1$ then by Lemma \ref{space_size}, $\mathcal X_1 = [a_1, b_1]^d$ where $b_1 - a_1 = 2(b-a)$. By the transformation, the center of $\mathcal X_1$ only moves in the domain $\mathcal C_{initial}$. If we set $\mathcal C_T = [c_{min} - (b-a), c_{max} + (b-a)]^d$ then $\mathcal X_1 \subseteq \mathcal C_T$.

We assume that the statement is true for $T \ge 1$. By the inductive hypothesis, for every $1 \le t \le T$, $X_t \subseteq \mathcal C_T = [c_{min} - (b-a)(1+\sum_{j=1}^{T} j^{\alpha})/2, c_{max} + (b-a)(1+\sum_{j=1}^{T} j^{\alpha})/2]^d$. We set $\mathcal C_{T+1} = [c_{min} - (b-a)(1+\sum_{j=1}^{T+1} j^{\alpha})/2, c_{max} + (b-a)(1+\sum_{j=1}^{T+1} j^{\alpha})/2]^d$. First, we have $\mathcal C_{T} \subset \mathcal C_{T+1}$. Next we prove that $\mathcal X_{T+1} \subseteq \mathcal C_{T+1}$. Indeed, by Lemma \ref{space_size}, $X_{T+1} = [a_{T+1}, b_{T+1}]^d$ where $b_{T+1} - a_{T+1} = (b-a)(1+\sum_{j=1}^{T+1} j^{\alpha})$. By the transformation,  the center of $\mathcal X_{T+1}$ only moves in the domain $\mathcal C_{initial}$. It implies that $\mathcal X_{T+1}$ belongs to $\mathcal C_{T+1}$. The statement holds for any $T \ge 1$.
\end{proof}
\subsection{Properties of The Gamma Function and The Hyperharmonic Series}
\begin{lemma}(Lower bounds of a partial sum of a hyperharmonic series,\citep{Chlebus2009})
Given a partial sum of a hyperharmonic series $p_n = \sum_{j =1}^{n}j^{\alpha}$, where $n \in \mathbb{N}$. Then,
\begin{itemize}
  \item $p_n > \frac{(n+1)^{\alpha +1}-1}{\alpha +1}$ if $-1 \le \alpha < 0 $,
  \item $ p_n > ln(n+1)$ if $\alpha = -1$.
\end{itemize}
\label{le:1}
\end{lemma}

\begin{lemma}(Upper Bounds of a Hyperharmonic Series, \citep{Chlebus2009})
Given a hyperharmonic series $p_n = \sum_{j =1}^{n}j^{\alpha}$, where $n \in \mathbb{N}$. Then,
\begin{itemize}
  \item $p_n < 1 + \frac{n^{1 + \alpha}-1}{1 + \alpha}$ if $-1 \le \alpha < 0 $,
  \item $ p_n < 1 + ln(n)$ if $\alpha = -1$
\end{itemize}
\label{le:2}
\end{lemma}
\begin{lemma}(Bounding $p$-series when $p >1$, \citep{tom99})
Given a $p$-series $s_n = \sum_{k =1}^{n}\frac{1}{k^p}$, where $n \in \mathbb{N}$. Then,
$$s_n < \zeta(p) < \frac{1}{p-1} + 1$$
for any $n$, where $\zeta(p) = \sum_{k =1}^{\infty}\frac{1}{k^p} $ is Euler–Riemann zeta function that always converges. For example, $\zeta(3/2) \approx 2.61$, $\zeta(2) = \frac{\pi^2}{6}$.
\label{le:3}
\end{lemma}
\begin{lemma}
$\Gamma(\frac{d}{2} +1)^{\frac{1}{d}} < \sqrt{d+2}$
\label{le:4}
\end{lemma}
\begin{proof}
We consider two cases:
\begin{itemize}
  \item if $d = 2n$, where $n \in \mathbb{N}$ then $\Gamma(\frac{d}{2} +1) = \Gamma(n+1) = n!$
  \item if $d = 2n +1$, where $n \in \mathbb{N}$ then $\Gamma(\frac{d}{2} +1) = \Gamma(n+1 + \frac{1}{2}) = n!\Gamma(\frac{1}{2}) = \sqrt{\pi}n! < 2n!$
\end{itemize}
Hence, in both case, $\Gamma(\frac{d}{2} +1) < 2n!$. By Cauchy-Schwarz, we have:$n! < (\frac{1 + 2 + ... + n)}{n})^n = (\frac{n+1}{2})^n$.
However, $n \le \frac{d}{2}$. Thus, $(\Gamma(\frac{d}{2} +1))^{\frac{1}{d}} < 2(\frac{n+1}{2})^{\frac{n}{d}} < \sqrt{2(n+1)} < \sqrt{d+2}$.
\end{proof}
\section{Proof of Theorem 1}
\begin{theorem}[Reachability]
If $\alpha \ge -1$, then the HuBO algorithm guarantees that there exists a constant $T_0 > 0$ (independent of $t$) such that when $t > T_0$,  $\mathcal X_t$ contains $\mathbf{x}^*$.
\label{theorem:1}
\end{theorem}
We denote the center of the user-defined finite region $\mathcal C_{initial} = [c_{min},c_{max}]^d$ as $\mathbf{c}_0$. By the assumption of $\mathbf{x}^*$  being not at infinity, there exists a smallest range $[a_g, b_g]^d$ so that both $\mathbf{x}^*$ and  $\mathbf{c}_0$ belong to  $[a_g, b_g]^d$. By induction, the search space $\mathcal X_t$ at iteration $t$ is a hypercube, denoted by $[a_t, b_t]^d$. Following our search space expansion, the center of $\mathcal X_t$ only moves in region $\mathcal C_{initial}$. Therefore, for each dimension $i$,  we have in the worst case,  $(b_t - [\mathbf{c}_0]_i)$ is at least $\frac{c_{min}-c_{max}}{2} + \frac{b_t -a_t}{2}$ and $([\mathbf{c}_0]_i - a_t)$ is at most $\frac{c_{max} -c_{min}}{2} - \frac{b_t - a_t}{2}$. By Lemma \ref{space_size}, the size of $\mathcal X_t$ as $b_t - a_t: b_t- a_t = (b-a)(1+\sum_{j=1}^{t} j^{\alpha})$.
Therefore, $(b_t - [\mathbf{c}_0]_i)$ is at least $ \frac{c_{min}-c_{max}}{2} + \frac{b-a}{2}(1+\sum_{j=1}^{t} j^{\alpha})$ and $([\mathbf{c}_0]_i - a_t)$ is at most $\frac{c_{max}-c_{min}}{2} - \frac{b-a}{2}(1+\sum_{j=1}^{t} j^{\alpha})$.

If there exists a $T_0$ such that two conditions satisfy: (1) $\frac{c_{min}-c_{max}}{2} + \frac{b-a}{2}(1+\sum_{j=1}^{T_0} j^{\alpha}) \ge b_g$, and (2) $\frac{c_{max}-c_{min}}{2} - \frac{b-a}{2}(1+\sum_{j=1}^{T_0} j^{\alpha})  < a_g$, then we can guarantee that for all $t>T_0$, the search space $\mathcal X_t$ will contain $[a_g, b_g]^d$ and thus also contain $\mathbf{x}^*$.

Such a $T_0$ exists because $p_t = \sum_{j=1}^{t} j^{\alpha}$ is a diverging sum with $t$ when $\alpha \ge -1$. Indeed,
by Lemma \ref{le:1}, we have
\begin{itemize}
  \item $p_t > \frac{(t+1)^{\alpha +1}-1}{\alpha +1}$ if $-1 \le \alpha < 0 $,
  \item $ p_t > ln(t+1)$ if $\alpha = -1$.
\end{itemize}
For both cases, $\lim_{t \rightarrow \infty}p_t \rightarrow \infty$. From the conditions (1) and (2), we can see that $T_0$ is a function of parameters $a$,  $b$, $c_{min}$, $c_{max}$, $\alpha$ and $a_g, b_g$. Since $a$, $b$, $c_{min}$, $c_{max}$ and $\alpha$ are determined at the beginning of the HuBO algorithm and do not change, such a constant (although unknown) $T_0$ exists. \qedhere
\section{Proof of Theorem 2}
To derive an upper bound of the cumulative regret of the HuBO algorithm for SE kernels and Mat\'ern kernels, we first derive an upper bound of the cumulative regret for a general class of kernels according to the maximum information gain. We do this in the following Proposition 1. Next, we provide upper bounds for the maximum information gain on SE kernels and  Mat\'ern kernels. We do that in Proposition 2. Finally, we prove the correctness of Theorem 2 by combining Proposition 1 and Proposition 2.
\begin{proposition}
Let $f \sim \mathcal{GP}(\mathbf{0}, k)$ with a stationary covariance function $k$. Assume that $-1 \le \alpha < 0$ and there exist constants $s_1, s_2 > 0$ such that $\mathbb{P}[sup_{\mathbf{x} \in \mathcal X}|\partial f/ \partial x_{i}| > L] \le s_1e^{-(L/s_2)^2}$
for all $L > 0$ and for all $i \in \{1,2,..., d\}$. Pick a $\delta \in (0,1)$. Set $\beta_T = 2log(4\pi_t/\delta) + 4dlog(dTs_2(b-a)(1 + \sum_{j=1}^{T} j^{\alpha})\sqrt{log(4ds_1/\delta)})$. Thus, there is a constant $C$ such that for any horizon $T > T_0$, the cumulative regret of the proposed HuBO algorithm is bounded as
$$R_T \le C  + \sqrt{C_1T\beta_T\gamma_T(\mathcal C_T)} + \frac{\pi^2}{6}$$,
with probability $1 -\delta$, where the domain $\mathcal C_T = [c_{min} - (b-a)(1+\sum_{j=1}^{T} j^{\alpha})/2, c_{max} + (b-a)(1+\sum_{j=1}^{T} j^{\alpha})/2]^d$, and $\gamma_T(\mathcal C_T)$ is the maximum information gain  for any $T$ observations in the domain $\mathcal C_T$ (see \citep{Srinivas12}).
\end{proposition}
\begin{proof}
Let us denote by $f^*_t$ the optimum in the search space $\mathcal X_t$, and denote by $g_t$ the gap between the global optimum and the optimum in the search space $\mathcal X_t$. Formally,  $g_t = f(x^*) - f^*_t$.
We consider two cases:
\begin{itemize}
  \item $1 \le t \le T_0$. Then with the probability $1-\delta$, $r_t$ can be bounded as follows:
  \begin{eqnarray}
    r_t & = & f(x^*)- f(x_t) \\
    & = & f_t^* -f(x_t) + g_t \\
    & = & f_t^* - \mu_{t-1}(x^*) + \mu_{t-1}(x^*) - f(x_t) + g_t \\
    & \le & f(x^*) - \mu_{t-1}(x^*) + \mu_{t-1}(x^*) - f(x_t) + g_t \\
    & \le & \sqrt{\beta_t}\sigma_{t-1}(x^*) +  \mu_{t-1}(x^*) - f(x_t) + g_t\\
    & \le & \sqrt{\beta_t}\sigma_{t-1}(x_t) +  \mu_{t-1}(x_t) - f(x_t) + g_t\\
    & \le & 2\sqrt{\beta_t}\sigma_{t-1}(x_t) + g_t
  \end{eqnarray}
where the inequality (4) holds as $f^*_t \le f(x^*)$, the inequality (5) holds as $f(x^*) \le \mu_{t-1}(x^*) + \sqrt{\beta_t}\sigma_{t-1}(x^*)$ with probability $1 -\delta$ ( the proof is similar to Lemma 5.5 of \citep{Srinivas12}), the inequality (6) holds as $\sqrt{\beta_t}\sigma_{t-1}(x^*) +  \mu_{t-1}(x^*)  = \mu_t(x^*) \le \sqrt{\beta_t}\sigma_{t-1}(x_t) +  \mu_{t-1}(x_t) = \mu_t(x_t)$ ( recall that $x_t = \text{argmax}_{x \in \mathcal X_t} u_t(x)$), and finally inequality (7) holds as $\mu_{t-1}(x_t) - \sqrt{\beta_t}\sigma_{t-1}(x_t) \le f(x_t)$ with probability $1 -\delta$ ( the proof is similar to Lemma 5.1 of \citep{Srinivas12}).

  \item $t > T_0$. By Theorem 1, the search space $\mathcal X_t$ contains $x^*$. Similar to the idea of \citep{Srinivas12}, we can use a set of \emph{discretizations} of $\mathcal X_t$ to achieve a valid confidence interval on $x^*$. By proof similar to Lemma 5.8 of \citep{Srinivas12}, we achieve: $r_t \le 2\sqrt{\beta_t}\sigma_{t-1}(x_t) + \frac{1}{t^2}$.
\end{itemize}
Combining the two cases,  we achieve $R_T = \sum_{t=1}^{T}r_t \le \sum_{t=1}^{T_0}g_t + 2\sum_{t=1}^{T}\sqrt{\beta_t}\sigma_{t-1}(x_t) + \sum_{t=T_0 +1}^{T} \frac{1}{t^2} \le C + 2\sum_{t=1}^{T}\sqrt{\beta_t}\sigma_{t-1}(x_t) + \sum_{t=1}^{T} \frac{1}{t^2} \le C + 2\sum_{t=1}^{T}\sqrt{\beta_t}\sigma_{t-1}(x_t) + \sum_{t=1}^{T} \frac{1}{t^2} + \frac{\pi^2}{6}$, where we set $C = \sum_{t=1}^{T_0}g_t$. To make our problem in context of unknown search spaces tractable, we assume that the function $f$ is \emph{finite} on any finite domain of $\mathbb{R}^d$. It implies that for every $1 \le t \le T_0$, $g_t$ is finite. Further, by definition of $T_0$, $T_0$ is the constant and independent of $T$. Thus, $C$ is also a constant and is independent of $T$.

Next, we derive an upper bound on $\sum_{t=1}^{T}\sqrt{\beta_t}\sigma_{t-1}(x_t)$.  By Lemma \ref{volume_size}, for every $1 \le t \le T$, $\mathcal X_t \subseteq \mathcal C_T$. Similar to the proof of Lemma 5.4 of \citep{Srinivas12} we can achieve
\begin{eqnarray}
\sum_{t=1}^{T}4\beta_t\sigma^2_{t-1}(x_t) \le C_1\beta_T\gamma_T(\mathcal C_T),
\end{eqnarray}
where  $C_1 = 8/log(1 + \sigma^2) $, $\beta_t = 2log(4\pi_t/\delta) + 4dlog(dts_2(b-a)(1 + \sum_{j=1}^{t} j^{\alpha})\sqrt{log(4ds_1/\delta)})$.

By Cauchy-Schwarz, we have:
\begin{eqnarray}
\sum_{t=1}^{T}\sqrt{\beta_t}\sigma_{t-1}(x_t) \le \sqrt{C_1T\beta_T\gamma_T(\mathcal C_T)}
\end{eqnarray}
Therefore,  $R_T \le C + \sqrt{C_1T\beta_T\gamma_T(\mathcal C_T)} + \frac{\pi^2}{6}$.
\end{proof}
\begin{proposition}
We assume the kernel function $k$ satisfies $k(x,x') \le 1$. Then,
\begin{itemize}
  \item For SE kernels: $\gamma_T(\mathcal C_T) = \mathcal O(T^{(\alpha+1)d})$,
  \item For Mat\'ern kernels with $\nu > 1$: $\gamma_T(\mathcal C_T) = \mathcal O(T^{\frac{d^2(\alpha +2) + d}{2\nu + d(d+1)}})$
\end{itemize}
\end{proposition}
\begin{proof}
For SE kernels, by the proof similar as in Theorem 5 of \citep{Srinivas12}, we can bound $\gamma_T(\mathcal C_T)$ as $\gamma_T(\mathcal C_T) \le \mathcal O(Vol(\mathcal C_T)log(T))$. By definition, $\mathcal C_T = [c_{min} - (b-a)(1+\sum_{j=1}^{T} j^{\alpha})/2, c_{max} + (b-a)(1+\sum_{j=1}^{T} j^{\alpha})/2]^d$. Hence, $Vol(\mathcal C_T) = (c_{max} - c_{min} + (b-a)( 1 + \sum_{j=1}^{T}j^{\alpha}))^d$. We consider two cases on $\alpha$:
\begin{itemize}
  \item $\alpha =-1$. By Lemma \ref{le:2}, $\sum_{j=1}^{T}j^{\alpha} < 1 + ln(T)$. Hence,  $Vol(\mathcal C_T) < (c_{max} - c_{min} + (b-a)(2 + ln(T)))^d$. Therefore, $\gamma_T(\mathcal C_T) \le \mathcal O((ln(T))^{d+1})$.
  \item if $-1 < \alpha < 0$. By Lemma \ref{le:2}, $\sum_{j=1}^{T}j^{\alpha} < 1 + \frac{T^{1+\alpha}-1}{1+\alpha}$. Hence, $Vol(\mathcal C_T) = (c_{max} - c_{min} + (b-a)( 1 + \sum_{j=1}^{T}j^{\alpha}))^d < (c_{max} - c_{min} + \frac{b-a}{(1 + \alpha)^d}( 2\alpha +1 + T^{\alpha +1}))^d$. Thus, $Vol(\mathcal C_T) = \mathcal O(T^{(\alpha+1)d})$.
\end{itemize}
Thus, $\gamma_T(\mathcal C_T) = \mathcal O(T^{(\alpha+1)d})$.

For Mat\'ern kernels, by the proof similar as in Theorem 5 of \citep{Srinivas12}, we can bound $\gamma_T(\mathcal C_T)$ as $\gamma_T(\mathcal C_T) = \mathcal O(T_{*}log(Tn_T)))$, where $n_T = 2Vol(\mathcal C_T)(2\tau +1) T^{\tau}(log T)$  and $T_{*} = (Tn_T)^{d/(2\nu +d)}(log(Tn_T))^{-d/(2\nu + d)}$, $\tau$ is a parameter. We consider two cases:
\begin{itemize}
  \item if $\alpha =-1$, $\sum_{j=1}^{T}j^{\alpha} < 1 + ln(T)$. We have $Vol(\mathcal C_T) = \mathcal O(log T)$ and $\mathcal O(T_{*}log(Tn_T)) = \mathcal O(T^{\frac{(\tau +1)d}{2\nu + d}}(log T)$. We choose $\tau = \frac{2\nu d}{2\nu + d(d+1)}$ to match this term with $\mathcal O(T^{1 -\frac{\tau}{d}})$. Thus, $\gamma_T(\mathcal C_T) = \mathcal O(T^{1 - \frac{\tau}{d}}) = \mathcal O(T^{\frac{d(d+1)}{2\nu + d(d+1)}})$.
  \item if $-1 < \alpha < 0$, we obtain $Vol(\mathcal C_T) = \mathcal O(T^{(\alpha+1)d})$. Thus, $\mathcal O(T_{*}log(Tn_T)) = \mathcal O(T^{\frac{((\alpha +1)d + \tau +1)d}{2\nu + d}}(log T)$. To match this term with $\mathcal O(T^{1- \frac{\tau}{d}})$ we choose $\tau$ such that:
            $$\frac{((\alpha +1)d + \tau +1)d}{2\nu + d} = 1- \frac{\tau}{d}$$
      This is equivalent to $\tau = \frac{2\nu d - d^3(\alpha +1)}{2\nu + d(d+1)}$. Thus, $\gamma_T(\mathcal C_T) = \mathcal O(T^{1 - \frac{\tau}{d}}) = \mathcal O(T^{\frac{d^2(\alpha +2) + d}{2\nu + d(d+1)}})$.
\end{itemize}
Since, when $\alpha = -1$, $\gamma_T(\mathcal C_T) = \mathcal O(T^{\frac{d^2(\alpha +2) + d}{2\nu + d(d+1)}}) =\mathcal O(T^{\frac{d(d+1)}{2\nu + d(d+1)}})$. Thus, we can write $\gamma_T(\mathcal C_T) = \mathcal O(T^{\frac{d^2(\alpha +2) + d}{2\nu + d(d+1)}})$ for $-1 \le \alpha < 0$.
\end{proof}
Combining Proposition 1 and Proposition 2, we achieve Theorem 2.
\begin{theorem}[Cumulative Regret $R_T$ of HuBO Algorithm]
Let $f \sim \mathcal{GP}(\mathbf{0}, k)$ with a stationary covariance function $k$. Assume that there exist constants $s_1, s_2 > 0$ such that $\mathbb{P}[sup_{\mathbf{x} \in \mathcal X}|\partial f/ \partial x_{i}| > L] \le s_1e^{-(L/s_2)^2}$
for all $L > 0$ and for all $i \in \{1,2,..., d\}$. Pick a $\delta \in (0,1)$. Thus, if  $-1 \le \alpha < 0$ then for any horizon $T > T_0$, the cumulative regret of the proposed HuBO algorithm is bounded as
\begin{itemize}
  \item \scalebox{0.9}{$R_T \le \mathcal O^*(T^{\frac{(\alpha +1)d +1}{2}})$} if $k$ is a SE kernel,
  \item \scalebox{0.9}{$R_T \le \mathcal O^*(T^{\frac{d^2(\alpha +2) + d}{4\nu + 2d(d+1)}})$} if $k$ is a Mat\'ern kernel
\end{itemize}
with probability greater than $1 -\delta$.
\label{cumulative_regret_1}
\end{theorem}
\begin{proof}
By Proposition 1, we have $R_T \le C  + \sqrt{C_1T\beta_T\gamma_T(\mathcal C_T)} + \frac{\pi^2}{6}$, where $\beta_T  = 2log(4\pi_t/\delta) + 4dlog(dTs_2(b-a)(1 + \sum_{j=1}^{T} j^{\alpha})\sqrt{log(4ds_1/\delta)})$. By Lemma \ref{le:2}, if $\alpha = -1$ then $\sum_{j=1}^{T}j^{\alpha} < 1 + ln(T)$, if $-1 < \alpha < 0$ then $\sum_{j=1}^{T}j^{\alpha} < 1 + \frac{T^{1+\alpha}-1}{1+\alpha}$. For both cases, $\beta_T \le \mathcal O(log(T))$. By Proposition 2, the Theorem 2 holds.
\end{proof}
\section{Proof of Theorem 3}
\begin{figure}[ht]
\begin{center}
\centerline{\includegraphics[scale=1.0,width=.30\textwidth]{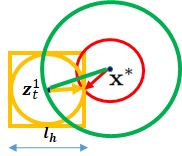}}
\caption{\emph{An illustration of the case where a hypercube (the yellow square) intersects the sphere $S_{\theta}$ (the red circle) in two-dimensional space. In this case, the inscribed sphere (the yellow circle centered at $z^1_t$ with the radius $\frac{l_h}{2}$) of the hypercube intersects the sphere $S_{\theta}$ since $z^1_t$ is within the circle centered at $x^*$ with the radius $\theta + \frac{l_h}{2}$ (the green circle).}}
\label{abc}
\end{center}
\end{figure}
\begin{theorem}
Pick a $\delta \in (0,1)$. Let $\mathbf{x}^*_t \in \mathcal H_t$ be the closest point to $\mathbf{x}^*$ in the search space $\mathcal H_t$. For any $t > T_0$ and $-1 \le \alpha < 0$, with probability greater than $1- \delta$, we have
\begin{equation}
||\mathbf{x}^*_t - \mathbf{x}^*||_2 < \frac{2(b-a)}{\pi}(\Gamma(\frac{d}{2} +1))^{\frac{1}{d}}(log(\frac{1}{\delta}))^{\frac{1}{d}}M_t,
\end{equation}
where the constant $T_0$ is defined in Theorem \ref{theorem:11}, $\Gamma$ is the gamma function, and $ M_t= (2 + ln(t))t^{-\frac{\lambda}{d}}$ if $\alpha = -1$, otherwise, $ M_t = 2(\alpha +1)^{-1}t^{-\frac{\lambda}{d}}$ if $-1 < \alpha < 0$.
\label{reachbility2}
\end{theorem}
\begin{proof}
The proof idea is to estimate the probability that $x^*_t$ lies in a sphere around $x^*$ with a small radius. Formally, we seek to bound  $\mathbb{P}[||x^*_t - x^*||_2 \le \theta]$ given a small $\theta > 0$.

It is hard to estimate directly $\mathbb{P}[||x^*_t - x^*||_2 \le \theta]$. Instead, our idea is as follows. Since $x^*_t \in \mathcal H_t$, there exists a hypercube which contains $x^*_t$. We estimate the probability that this hypercube intersects the sphere $S_{\theta} =\{x \in \mathbb{R}^d |||x-x^*||_2 \le \theta\}$ which is centered at the optimum $x^*$ with the radius $\theta$.

We consider the case of $t > T_0$. By Theorem 1, $\mathcal X_t$ contains $x^*$ for every $t > T_0$, where $\mathcal X_t$ is the search space of the HuBO algorithm . We recall that $\mathcal X_t$ is different from $\mathcal H_t$ which is the search space of the HD-HuBO algorithm that we are considering in this section. However, since
the search space $\mathcal H_t$ is defined via $\mathcal X_t$, we need to use $\mathcal X_t$ to bound $\mathcal H_t$.

There are two cases to consider: \textbf{Case 1}: the whole sphere $S_{\theta}$ is within $\mathcal X_t$; \textbf{Case 2}: the only part of $S_{\theta}$ is within $\mathcal X_t$. Note that it is impossible that the whole sphere $S_{\theta}$ is outside of $\mathcal X_t$ since at least we have $x^* \in \mathcal X_t$ for $t > T_0$.
\begin{figure}[ht]
\begin{center}
\centerline{\includegraphics[scale=1.0,width=.35\textwidth]{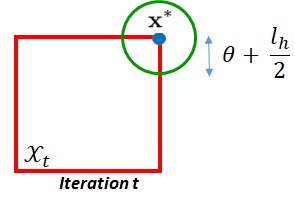}}
\caption{\emph{An illustration of the case where the global optimum $x^*$ is a vertex of the square $\mathcal X_t$ in two-dimensional space.
In this case,  only a 1/4 volume of the sphere $S_{\theta + \frac{l_h}{2}}$ centered at $x^*$  with the radius $\theta + \frac{l_h}{2}$ (the green circle) is inside of $\mathcal X_t$.}}
\label{fig:b}
\end{center}
\vskip -0.2in
\end{figure}
\begin{itemize}
  \item \textbf{Case 1} where the whole sphere $S_{\theta}$ is within $\mathcal X_t$. We seek to bound the probability that a hypercube $H(z^i_t, l_h)$ intersects the sphere $S_{\theta}$, $1 \le i \le N_t$. We denote this probability by $p_0$. This probability is greater than the probability that the inscribed sphere of the hypercube $H(z^i_t, l_h)$, denoted by $S(z^i_t, \frac{l_h}{2})$ that has the center at $z^i_t$ and the radius $\frac{l_h}{2}$ intersects the sphere $S_{\theta}$. Let us define this probability as $p_1$. Further, $p_1$ is greater than the probability that the point $z^i_t$ is within the sphere around $x^*$ with the radius $\theta + \frac{l_h}{2}$. Let us define this probability as $p_2$. To explain the connection  $p_1 \ge p_2$, we can see that the condition so that the sphere $S(z^i_t, \frac{l_h}{2})$ intersects the sphere $S_{\theta}$ is the distance between two centers $x^*$ and $z^i_t$ is less than or equal to the total of two radius. Figure \ref{abc} illustrates our situation.

    The probability $p_2$ can be computed by
    $$\frac{Vol(S_{\theta + \frac{l_h}{2}})}{Vol(\mathcal X_t)},$$
    where $Vol(\mathcal X_t)$ denotes the volume of the $\mathcal X_t$ and  $Vol(S_{\theta + \frac{l_h}{2}})$ denotes the volume of the sphere $S_{\theta + \frac{l_h}{2}}$ centered at $x^*$ with the radius $\theta + \frac{l_h}{2}$.

    Since $p_0 > p_2$, we achieve
    $$p_0 > \frac{Vol(S_{\theta + \frac{l_h}{2}})}{Vol(\mathcal X_t)}.$$
    By Lemma \ref{space_size},  the volume of $\mathcal X_t$ can be computed as $Vol(\mathcal X_t) = ((b-a)(1 + \sum_{j=1}^{t} j^{\alpha}))^d$. By \citep{Xian05}, the volume of the $d$-dimensional sphere with radius $\theta + \frac{l_h}{2}$ in $L^2$ norms is $\frac{(\pi (\theta + \frac{l_h}{2}))^d}{\Gamma(\frac{d}{2} +1)}$. Thus, the probability can be re-write as
    $$\frac{1}{\Gamma(\frac{d}{2} +1)} [\frac{2\Gamma(\frac{3}{2})(\theta + \frac{l_h}{2})}{(b-a)(1 + \sum_{j=1}^{t} j^{\alpha}))}]^d.$$
  \item \textbf{Case 2} where the only part of $S_{\theta}$ is within $\mathcal X_t$. We only consider the case where $\theta + \frac{l_h}{2} < b-a$. Note that $b-a$ is the size of the initial space $\mathcal X_0$ as defined in Algorithm 1. $l_h$ denotes the size of hypercubes and $l_h$ is a parameter of the HD-HuBO algorithm. Hence, we can choose $l_h$ so that $l_h < b-a$ and $\theta < b-a - \frac{l_h}{2}$. It means that the sphere $S_{\theta}$ is small compared to $\mathcal X_t$.

    In the worst case where $x^*$ is at the boundary of $\mathcal X_t$ for all dimensions. See Figure \ref{fig:b} for an explanation. In this case, the size of the space part of $S_{\theta + \frac{l_h}{2}}$ in $\mathcal X_t$ halves in each dimension and therefore, the volume of the space part of $S_{\theta + \frac{l_h}{2}}$ in $\mathcal X_t$, represented by $Vol(S_{\theta + \frac{l_h}{2}}) \cap Vol(\mathcal X_t)$ is reduced by $2^d$ times, compared to the whole volume of the sphere $S_{\theta + \frac{l_h}{2}}$. Thus, similar to Case 1, the probability $p_0$ that a hypercube $H(z^i_t, l_h)$ intersects the sphere $S_{\theta}$ is bounded as
    $$p_0 > \frac{Vol(S_{\theta + \frac{l_h}{2}}) \cap Vol(\mathcal X_t)}{Vol(\mathcal X_t)} = \frac{1}{2^d}\frac{1}{\Gamma(\frac{d}{2} +1)} [\frac{2\Gamma(\frac{3}{2})(\theta + \frac{l_h}{2})}{(b-a)(1 + \sum_{j=1}^{t} j^{\alpha}))}]^d.$$
\end{itemize}
Thus, in both Case 1 and Case 2, we have that the probability that a hypercube $H(z^i_t, l_h)$ intersects the sphere $S_{\theta}$ is bounded as
$$p_0 > \frac{1}{2^d}\frac{1}{\Gamma(\frac{d}{2} +1)} [\frac{2\Gamma(\frac{3}{2})(\theta + \frac{l_h}{2})}{(b-a)(1 + \sum_{j=1}^{t} j^{\alpha}))}]^d.$$
It implies that the probability that a hypercube $H(z^i_t, l_h)$ does not intersect the sphere $S_{\theta}$ is computed as
\begin{eqnarray*}
1 - p_0 & < & 1 -\frac{1}{2^d}\frac{1}{\Gamma(\frac{d}{2} +1)} [\frac{2\Gamma(\frac{3}{2})(\theta + \frac{l_h}{2})}{(b-a)(1 + \sum_{j=1}^{t} j^{\alpha}))}]^d \\
& = &  1 - \frac{1}{\Gamma(\frac{d}{2} +1)} [\frac{\Gamma(\frac{3}{2})(\theta + \frac{l_h}{2})}{(b-a)(1 + \sum_{j=1}^{t} j^{\alpha}))}]^d \\
& < &  e^{-\frac{1}{\Gamma(\frac{d}{2} +1)} [\frac{\Gamma(\frac{3}{2}) (\theta + \frac{l_h}{2})}{(b-a)(1 + \sum_{j=1}^{t} j^{\alpha}))}]^d},
\end{eqnarray*}
where we use the inequality $1-x \le e^{-x}$.

Therefore, if we consider the set of $N_t$ hypercubes then the probability that no hypercube $H(z^i_t, l_h)$ intersects the sphere $S_{\theta}$ is less than
$$ \prod_{1 \le i \le N_t} e^{-\frac{1}{\Gamma(\frac{d}{2} +1)} [\frac{\Gamma(\frac{3}{2}) (\theta + \frac{l_h}{2})}{(b-a)(1 + \sum_{j=1}^{t} j^{\alpha}))}]^d} = e^{- N_t\frac{1}{\Gamma(\frac{d}{2} +1)} [\frac{\Gamma(\frac{3}{2}) (\theta + \frac{l_h}{2})}{(b-a)(1 + \sum_{j=1}^{t} j^{\alpha}))}]^d}$$
Note that this is achieved because the set of centres of hypercubes is sampled uniformly at random (hence independently). Thus, the probability that there is at least a hypercube from the set of $N_t$ hypercubes which intersects the sphere $S_{\theta}$ is at least:
$$1 - e^{- N_t\frac{1}{\Gamma(\frac{d}{2} +1)} [\frac{\Gamma(\frac{3}{2}) (\theta + \frac{l_h}{2})}{(b-a)(1 + \sum_{j=1}^{t} j^{\alpha}))}]^d}.$$

Further, since $l_h \ge 0$, $1 - e^{- N_t\frac{1}{\Gamma(\frac{d}{2} +1)} [\frac{\Gamma(\frac{3}{2}) (\theta + \frac{l_h}{2})}{(b-a)(1 + \sum_{j=1}^{t} j^{\alpha}))}]^d} \ge 1 - e^{- N_t\frac{1}{\Gamma(\frac{d}{2} +1)} [\frac{\Gamma(\frac{3}{2}) (\theta)}{(b-a)(1 + \sum_{j=1}^{t} j^{\alpha}))}]^d}$. Thus, the probability that there is at least a hypercube from the set of $N_t$ hypercubes which intersects the sphere $S_{\theta}$ is greater than:
$$1 - e^{- N_t\frac{1}{\Gamma(\frac{d}{2} +1)} [\frac{\Gamma(\frac{3}{2}) (\theta)}{(b-a)(1 + \sum_{j=1}^{t} j^{\alpha}))}]^d}.$$
Note that here, we omit the influence of the size of hypercubes. In fact, the larger the $l_h$, the higher the probability that there is at least a hypercube from the set of $N_t$ hypercubes which intersects the sphere $S_{\theta}$.

On the other hand, if let $x^*_t \in \mathcal H_t$ be the closest point to $x^*$ in the search space $\mathcal H_t$
then the probability that there is at least a hypercube from the set of $N_t$ hypercubes which intersects the sphere $S_{\theta}$ is equal to the probability that $||x^*_t - x^*||_2 \le \theta$. Thus, we have
$$\mathbb{P}[||x^*_t - x^*||_2 \le \theta] > 1 - e^{- N_t\frac{1}{\Gamma(\frac{d}{2} +1)} [\frac{\Gamma(\frac{3}{2}) (\theta)}{(b-a)(1 + \sum_{j=1}^{t} j^{\alpha}))}]^d}.$$
Now set $ e^{- N_t\frac{1}{\Gamma(\frac{d}{2} +1)} [\frac{\Gamma(\frac{3}{2}) (\theta)}{(b-a)(1 + \sum_{j=1}^{t} j^{\alpha}))}]^d}= \delta$. We achieve $\theta = \frac{(b-a)}{\Gamma(\frac{3}{2})}(1 + \sum_{j=1}^{t} j^{\alpha}) (\Gamma(\frac{d}{2} +1))^{\frac{1}{d}}(\frac{1}{N_t}log(\frac{1}{\delta}))^{\frac{1}{d}} = \frac{2(b-a)}{\sqrt{\pi}}(1 + \sum_{j=1}^{t} j^{\alpha}) (\Gamma(\frac{d}{2} +1))^{\frac{1}{d}}(\frac{1}{N_t}log(\frac{1}{\delta}))^{\frac{1}{d}}$. Here, we use $\Gamma(\frac{3}{2}) = \frac{\sqrt{\pi}}{2}$.

Thus, given a $\delta \in (0,1)$, we have
\begin{eqnarray*}
||x_t^*- x^*||_2  < \frac{2(b-a)}{\sqrt{\pi}} (1 + \sum_{j=1}^{t} j^{\alpha})(\Gamma_t(\frac{d}{2} +1))^{\frac{1}{d}}(\frac{1}{N_t}log(\frac{1}{\delta}))^{\frac{1}{d}},
\end{eqnarray*}
with the probability $1 - \delta$.

By definition, $N_t = N_0\ceil*{t^{\lambda}} \ge t^{\lambda}$. Using the results from Lemma 2, we consider two cases of $\alpha$:
\begin{itemize}
  \item if $\alpha =-1$, $1 + \sum_{j=1}^{t} \frac{1}{j} < 2 + ln(t)$. In this case, $||x_t^*- x^*||_2  \le \frac{2(b-a)}{\sqrt{\pi}}(2 + ln(t)) (\Gamma(\frac{d}{2} +1))^{\frac{1}{d}}(\frac{1}{N_t}log(\frac{1}{\delta}))^{\frac{1}{d}}  \le \frac{2(b-a)}{\sqrt{\pi}}(\Gamma(\frac{d}{2} +1))^{\frac{1}{d}}(log(\frac{1}{\delta}))^{\frac{1}{d}} \frac{2 + ln(t)}{t^{\frac{\lambda}{d}}}$.
  \item if $-1 < \alpha <0$, $1 + \sum_{j=1}^{t} j^{\alpha} < 2 + \frac{t^{\alpha +1}-1}{\alpha +1} = \frac{t^{\alpha +1} + 2\alpha +1}{1 + \alpha}$. Since $\alpha \le 0$, $\frac{t^{\alpha +1} + 2\alpha + 1}{1 + \alpha} \le \frac{t^{\alpha +1} + 1}{1 + \alpha} \le \frac{2}{\alpha +1}$. Thus, $||x_t^*- x^*||_2  \le \frac{2(b-a)}{\sqrt{\pi}}(\Gamma(\frac{d}{2} +1))^{\frac{1}{d}}(log(\frac{1}{\delta}))^{\frac{1}{d}} \frac{2}{\alpha +1}t^{-\frac{\lambda}{d}}$.
\end{itemize}
Let \begin{equation*}
  M_t=\begin{cases}
    (2 + ln(t))t^{-\frac{\lambda}{d}}, & \text{if $\alpha = -1$}.\\
   \frac{2}{\alpha +1}t^{-\frac{\lambda}{d}}, & \text{if $-1 < \alpha < 0$}.
  \end{cases}
\end{equation*},
we have $$||x^*_t - x^*||_2 < \frac{2(b-a)}{\sqrt{\pi}}(\Gamma(\frac{d}{2} +1))^{\frac{1}{d}}(log(\frac{1}{\delta}))^{\frac{1}{d}}M_t,$$
with the high probability $1- \delta$. The Theorem holds.
\end{proof}
\section{Proof of Theorem 4}
Similar to HuBO, to derive the upper bounds of the cumulative regret of HD-HuBO for SE kernels and Mat\'ern kernels, we first derive an upper bound of the cumulative regret for a general class of kernels as the following Proposition 3. We use Theorem 3 to prove this. Next, by combining results from Proposition 2 and Proposition 3, we achieve upper bounds for HD-HuBO for SE kernels and Mat\'ern kernels.
\begin{proposition}
Let $f \sim \mathcal{GP}(\mathbf{0}, k)$ with a stationary covariance function $k$. Assume that there exist constants $s_1, s_2 > 0$ such that $\mathbb{P}[sup_{\mathbf{x} \in \mathcal X}|\partial f/ \partial x_{i}| > L] \le s_1e^{-(L/s_2)^2}$
for all $L > 0$ and for all $i \in \{1,2,..., d\}$. Pick a $\delta \in (0,1)$. Set $\beta_t = 2log(\pi^2t^2/\delta) +  2dlog(2s_2l_hd\sqrt{log(6ds_1/\delta)}t^2)$. Then, there exists a constant $C'$ such that with any horizon $T > T_0$, under conditions $\frac{\lambda}{d} > \alpha +1$, $-1 \le \alpha < 0$, $l_h > 0$, the cumulative regret of HD-HuBO Algorithm is bounded with probability greater than $1 -\delta$ as

$R_T \le  C'  +  \sqrt{C_1T\beta_T\gamma_T(\mathcal C_T)} + A(log(\frac{6}{\delta}))^{\frac{1}{d}}B_T + \frac{\pi^2}{6}$
, where $A =  s_2\sqrt{log(\frac{l_h ds_1}{\delta})}\frac{2(b-a)}{\pi}d\sqrt{d+2}$, and where $B_T =U_T V_T$ such that $U_T = 2 + ln(T)$ if $\alpha =-1$, otherwise $U_T = 2(\alpha +1)^{-1}$, and $V_T = 1 + ln(T)$ if $\lambda = d$, otherwise $V_T = 1 + \frac{d}{d- \lambda} \text{max}\{ 1, T^{1 -\frac{\lambda}{d}}\}$,

$C_1 = 8/log(1 + \sigma^2)$, $l_h$ is the size of the hypercube, $\mathcal C_T = [c_{min} - (b-a)(1+\sum_{j=1}^{T} j^{\alpha})/2, c_{max} + (b-a)(1+\sum_{j=1}^{T} j^{\alpha})/2]^d$, and $\gamma_T(\mathcal C_T)$ is the
maximum information gain about the function $f$ from any $T$ observations from $\mathcal C_T$.
\label{R_T}
\end{proposition}

\paragraph{Our Idea} To derive a cumulative regret $R_T = \sum_{t=1}^{T} r_t$, we will seek to bound $r_t = f(x^*)- f(x_t)$ for any $t$.
If $t \le T_0$, similar to the proof of Proposition 2, we achieve a bound on $r_t$: $r_t \le 2\sqrt{\beta_t}\sigma_{t-1}(x_t) + g'_t$, where $\beta_t$ is defined as in section 5 in the main paper, $g'_t$ is the is the gap between the global optimum and the optimum in $\mathcal H_t$. Formally, $g'_t = f(x^*) - f^*(\mathcal H_t)$.

Now we consider the case where $t > T_0$. Let $x^*_t \in \mathcal H_t$ be the closest point to $x^*$ in the search space $\mathcal H_t$. To obtain a bound on $r_t$ ($t > T_0$), we write it as
\begin{eqnarray}
r_t & = & f(x^*)- f(x_t) \\
& = & \underbrace{f(x^*) - f(x^*_t)}_\text{Part 1} +  \underbrace{f(x^*_t)}_\text{Part 2} - \underbrace{f(x_t)}_\text{Part 3}
\end{eqnarray}

Now we start to bound the part 1, the part 2 and part 3.
\paragraph{Bounding Part 1}
\begin{lemma}
Pick a $\delta \in (0,1)$. For any $t > T_0$, with probability at least $1- \delta$, we have
$$|f(x^*) - f(x)| \le s_2\sqrt{log(\frac{2ds_1}{\delta})}\frac{2(b-a)}{\sqrt{\pi}}d \sqrt{d+2}(log(\frac{2}{\delta}))^{\frac{1}{d}}M_t,$$
where
\begin{equation*}
  M_t=\begin{cases}
    (2 + ln(t))t^{-\frac{\lambda}{d}}, & \text{if $\alpha = -1$}.\\
   \frac{2}{\alpha +1}t^{-\frac{\lambda}{d}}, & \text{if $-1 < \alpha < 0$}.
  \end{cases}
\end{equation*}
\label{lem:1}
\end{lemma}
\begin{proof}
Given any $x \in \mathcal X_t$,  by Assumption of Theorem 4 and the union bound, we have,
$$|f(x^*) - f(x)| \le L||x^* -x)||_1$$
with probability greater than $1 - ds_1e^{-L^2/s_2^2}$. Set $ds_1e^{-L^2/s_2^2} = \delta/2$. Thus,
\begin{eqnarray}
|f(x^*) - f(x)| \le s_2\sqrt{log(\frac{2ds_1}{\delta})}||x^* -x||_1
\label{eq:1}
\end{eqnarray}
with probability greater than $1 - \delta/2$.

On the other hand, By Theorem 3 we have:
\begin{eqnarray}
||x_t^*- x^*||_2  \le \frac{2(b-a)}{\sqrt{\pi}}(\Gamma(\frac{d}{2} +1))^{\frac{1}{d}}(log(\frac{2}{\delta}))^{\frac{1}{d}}M_t
\label{eq:2}
\end{eqnarray}
with probability $1 - \delta/2$, \begin{equation*}
  M_t=\begin{cases}
    (2 + ln(t))t^{-\frac{\lambda}{d}}, & \text{if $\alpha = -1$}.\\
   \frac{2}{\alpha +1}t^{-\frac{\lambda}{d}}, & \text{if $-1 < \alpha < 0$}.
  \end{cases}
\end{equation*}
To transform from the $L^2$ norms to the $L^1$ norms, we use Cauchy-Schwarz:
\begin{eqnarray}
||x_t^*-x^*||_1 \le d||x_t^*-x^*||_2
\label{eq:3}
\end{eqnarray}
Combining Eq(\ref{eq:1}), Eq(\ref{eq:2}) and  Eq(\ref{eq:3}), we have
$$|f(x^*) - f(x)| \le s_2\sqrt{log(\frac{2ds_1}{\delta})}d\frac{2(b-a)}{\sqrt{\pi}}(\Gamma(\frac{d}{2} +1))^{\frac{1}{d}}(log(\frac{2}{\delta}))^{\frac{1}{d}}M_t$$
with the probability $1 - \delta$.

Further, by Lemma \ref{le:4}, we achieve $(\Gamma(\frac{d}{2} +1))^{\frac{1}{d}} < \sqrt{d+2}$.  Thus, $$|f(x^*) - f(x)| \le s_2\sqrt{log(\frac{2ds_1}{\delta})}\frac{2(b-a)}{\sqrt{\pi}}d \sqrt{d+2}(log(\frac{2}{\delta}))^{\frac{1}{d}}M_t$$
with the probability $1 - \delta$.
\end{proof}

\paragraph{Bounding Part 2}
Now, we continue to bound the part 2. By definition, $x^*_t \in \mathcal H_t$. Since $\mathcal H_t = \{H(\mathbf{z}^{1}_t, l_h) \cup ...\cup H(\mathbf{z}^{N_t}_t, l_h)\} \cap \mathcal X_t$,  $x^*_t$ is in some hypercube. Without the loss of generality, we assume that $x^*_t$ is within the hypercube $H(z^*_t, l_h)$, where $z^*_t$ is one centre among sampled centres $\{z^1_t, ..., z^{N_t}_t\}$.
\begin{lemma}[Bounding Part 2]
Pick a $\delta \in (0,1)$ and set $\zeta_{t}^1 =  2log(\frac{\pi^2t^2}{3\delta}) +  2dlog(2s_1l_hd\sqrt{log(\frac{2ds_1}{\delta})}t^2)$.  Then, there exists a $x' \in H(z^*_t, l_h)$ such that
\begin{eqnarray}
f(x^*_t) \le  \mu_{t-1}(x') + \sqrt{\zeta_t^1}\sigma_{t-1}(x')  + \frac{1}{t^2}
\end{eqnarray}
holds with probability $\ge 1 - \delta$.
\label{lem:2}
\end{lemma}
\begin{proof}
We use the idea of proof of Lemma 5.7 in \citep{Srinivas12} for the hypercube $H(z^*_t, l_h)$. We consider the distance of any two points in the hypercube:
$||x - x'||_1$. We have $||x - x'||_1  \le l_h$, where $l_h$ is the size of the hypercube.

By Assumption of Theorem 4 and the union bound, for $\forall x, x'$, we have
$$|f(x) - f(x')| \le L||x -x'||_1$$
with probability greater than $1 - ds_1e^{-L^2/s_2^2}$.
Thus, by choosing $ds_1e^{-L^2/s_2^2} = \delta/2$, we have
\begin{eqnarray}
|f(x) - f(x')| \le s_2\sqrt{log(\frac{2ds_1}{\delta})}||x -x'||_1
\label{eq:4}
\end{eqnarray}
with probability greater than $1 - \delta/2$.

Now, on $H(z^*_t, l_h)$,  we construct a discretization $F_t$ of size $(\tau_t)^d$ dense enough such that for any $x \in F_t$ $$||x - [x]_t]||_1 \le \frac{l_h d}{\tau_t}$$
where $[x]_t$ denotes the closest point in $F_t$ to $x$. In this manner, with probability greater than $1 - \delta/2$, we have
\begin{eqnarray*}
|f(x) - f([x]_t)| & \le &  \frac{s_2\sqrt{log(\frac{2ds_1}{\delta})}||x- [x]_t||_1}{\tau_t} \\
& \le &  s_2\sqrt{log(\frac{2ds_1}{\delta})}\frac{l_h d}{\tau_t} \\
& < & \frac{s_2 l_h d\sqrt{log(\frac{2ds_1}{\delta})}}{\tau_t}
\end{eqnarray*}
Here, we use the inequality $||x- [x]_t||_1 \le l_h d$. Let $\tau_t = s_2l_h d\sqrt{log(\frac{2ds_1}{\delta})}t^2$. Thus, $|F_t| = (s_2l_hd\sqrt{log(\frac{2ds_1}{\delta})}t^2)^d$. We obtain
\begin{eqnarray}
|f(x) - f([x]_t)| \le \frac{1}{t^2}
\label{eq:4}
\end{eqnarray}
with probability $1- \delta/2$ for any $x \in F_t$.

Similar to Lemma 5.6 of \citep{Srinivas12}, if we set $\zeta_{t}^1 = 2log(|F_t|\frac{\pi^2t^2}{3\delta}) = 2log(\frac{\pi^2t^2}{3\delta}) +  2dlog(s_2l_hd\sqrt{log(\frac{2ds_1}{\delta})}t^2)$, we have with probability $1 -\delta/2$, we have
\begin{eqnarray}
f(x) \le \mu_{t-1}(x) + \sqrt{\zeta_{t}^1}\sigma_{t-1}(x)
\label{eq:5}
\end{eqnarray}
for any $x \in F_t$ and any $t \ge 1$. Thus, combining Eq(\ref{eq:4}) and Eq(\ref{eq:5}), if we let $[x]_t$ which is the closest point in $F_t$ to $x$, we have
\begin{eqnarray*}
f(x^*_t) & \le & \mu_{t-1}([x^*_t]_t) + \sqrt{\zeta_{t}^{1}}\sigma_{t-1}([x^*_t]_t) + \frac{1}{t^2}
\end{eqnarray*}
with probability $1 - \delta$.
\end{proof}
\paragraph{Bounding Part 3}
\begin{lemma}
Pick a $\delta \in (0,1)$ and set $\zeta_t^0 = 2log(\pi^2t^2/(6\delta))$. Then we have
\begin{eqnarray}
f(x_t) \ge \mu_{t-1}(x_t) - \sqrt{\zeta_t^0} \sigma_{t-1}(x_t)
\end{eqnarray}
holds with probability $\ge 1 - \delta$.
\label{lem:3}
\end{lemma}
\begin{proof}
It is similar to Lemma 5.5 of \citep{Srinivas12}.
\end{proof}

Now, we combine the results from Lemmas \ref{lem:1}, \ref{lem:2} and \ref{lem:3} to obtain a bound on $r_t$ as in the following Lemma.
\begin{lemma}[Bounding $r_t$]
Pick a $\delta \in (0,1)$ and set $\beta_t = 2log(\frac{\pi^2t^2}{\delta}) +  2dlog(2s_2l_hd\sqrt{log(\frac{6s_1d}{\delta})}t^2)$. Then with $t > T_0$, $\frac{\lambda}{d} > \alpha +1$, $-1 \le \alpha < 0$ and $l_h > 0$, we have
\begin{eqnarray}
r_t \le 2\beta_t^{1/2}\sigma_{t-1}(x_t) + \frac{1}{t^2} + A(log(\frac{6}{\delta}))^{\frac{1}{d}}M_t
\end{eqnarray}
holds with probability $\ge 1 - \delta$, where $A =  s_2\sqrt{log(\frac{2 ds_1}{\delta})}\frac{2(b-a)}{\sqrt{\pi}}d\sqrt{d+2}$ and
\begin{equation*}
  M_t=\begin{cases}
    (2 + ln(t))t^{-\frac{\lambda}{d}}, & \text{if $\alpha = -1$}.\\
   \frac{2}{\alpha +1}t^{-\frac{\lambda}{d}}, & \text{if $-1 < \alpha < 0$}.
  \end{cases}
\end{equation*}
\label{lem:4}
\end{lemma}
\begin{proof}
We use $\frac{\delta}{3}$ for Lemmas \ref{lem:1}, \ref{lem:2} and \ref{lem:3} so that these events hold simultaneously with probability greater than $1- \delta$. Formally, by Lemma \ref{lem:3} using $\frac{\delta}{3}$:
\begin{eqnarray*}
f(x_t) \ge \mu_{t-1}(x_t) - \sqrt{2log(\frac{\pi^2t^2}{\delta})}\sigma_{t-1}(x_t)
\end{eqnarray*}
holds with probability $\ge 1 - \frac{\delta}{3}$. As a result,
\begin{eqnarray}
f(x_t) & \ge &  \mu_{t-1}(x_t) - \sqrt{\zeta_t^0}\sigma_{t-1}(x_t) \\
& > & \mu_{t-1}(x_t) - \sqrt{\beta_t}\sigma_{t-1}(x_t) \label{eq:6}
\end{eqnarray}
holds with probability $\ge 1 - \frac{\delta}{3}$.

By Lemma \ref{lem:2} using $\frac{\delta}{3}$, there exists a $x' \in H(z^*_t, l_h)$ such that
\begin{eqnarray*}
f(x^*_t) \le \mu_{t-1}(x') + \sqrt{\zeta_t^{1}}\sigma_{t-1}(x') + \frac{1}{t^2}
\end{eqnarray*}
holds with probability $\ge 1 - \frac{\delta}{3}$. As a result,
\begin{eqnarray*}
f(x^*_t) & \le & \mu_{t-1}(x') + \sqrt{\zeta_t^{1}}\sigma_{t-1}(x') + \frac{1}{t^2} \\
f(x^*_t) & \le & \mu_{t-1}(x') + \sqrt{\beta_t}\sigma_{t-1}(x') + \frac{1}{t^2} \\
& = & u_t(x') + \frac{1}{t^2}
\end{eqnarray*}
Recall that $u_t(x)$ is the acquisition function  defined in the main paper.  Since $x_t = \text{argmax}_{x \in \mathcal X'_t}u_t(x)$ and $x' \in H(z^*_t, l_h) \subset \mathcal X'_t$, we have $u_t(x') \le u_t(x_t)$. Thus,
\begin{eqnarray}
 f(x^*_t) \le u_t(x_t) + \frac{1}{t^2} \label{eq:7}
\end{eqnarray}
holds with probability $\ge 1 - \frac{\delta}{3}$.

By Lemma \ref{lem:1} using $\frac{\delta}{3}$:
\begin{eqnarray}
|f(x_t^*) - f(x^*)| \le A(log(\frac{6}{\delta}))^{\frac{1}{d}}M_t
 \label{eq:8}
\end{eqnarray}
holds with probability $\ge 1 - \frac{\delta}{3}$.

Combing Eq(\ref{eq:6}), Eq(\ref{eq:7}) and Eq(\ref{eq:8}), we have
\begin{eqnarray}
r_t & = & f(x^*) - f(x_t)\\
& = & \underbrace{f(x^*) - f(x^*_t)}_\text{Part 1} +  \underbrace{f(x^*_t)}_\text{Part 2} - \underbrace{f(x_t)}_\text{Part 3}\\
& \le &  A(log(\frac{6}{\delta}))^{\frac{1}{d}}M_t  +  \underbrace{f(x^*_t)}_\text{Part 2} - \underbrace{f(x_t)}_\text{Part 3}\\
& \le & A(log(\frac{6}{\delta}))^{\frac{1}{d}}M_t + \frac{1}{t^2} + u_{t}(x_t)  - f(x_t) \label{eq:9}\\
& \le & A(log(\frac{6}{\delta}))^{\frac{1}{d}}M_t + \frac{1}{t^2} + 2(\beta_t)^{1/2}\sigma_{t-1}(x_t)
\label{eq:10}
\end{eqnarray}
holds with probability $\ge 1 - \delta$.
\end{proof}

Now we are ready to prove Proposition \ref{R_T}.

We have $R_T = \sum_{t=1}^{T} r_t = \sum_{t=1}^{T_0} r_t + \sum_{t=T_0 +1}^{T} r_t$.

Similar to the proof of Proposition 1, we have $\sum_{t=1}^{T_0} r_t  \le \sum_{t=1}^{T_0} (2\sqrt{\beta_t}\sigma_{t-1}(x_t) + g'_t)$.

On the other hand, By Lemma \ref{lem:4}, we have $ \sum_{t=T_0 +1}^{T} r_t \le \sum_{t=T_0 +1}^{T}( 2\beta_t^{1/2}\sigma_{t-1}(x_t) + \frac{1}{t^2} +  A(log(\frac{6}{\delta}))^{\frac{1}{d}}M_t)$. Thus,
$R_T = \sum_{t=1}^{T} r_t \le \sum_{t=1}^{T_0} g'_t + \frac{\pi^2}{6} + \sum_{t=1}^{T}2\beta_t^{1/2}\sigma_{t-1}(x_t) + \sum_{t=T_0 +1}^{T} A(log(\frac{6}{\delta}))^{\frac{1}{d}}M_t$. We set $C' = \sum_{t=1}^{T_0}g'_t$. To make our problem in context of unknown search spaces tractable, we assume that the function $f$ is \emph{finite} on any finite domain of $\mathbb{R}^d$. It implies that for every $1 \le t \le T_0$, $g_t'$ is finite. Further, by definition of $T_0$, $T_0$ is the constant and independent of $T$. Thus, $C'$ is also a constant and is independent of $T$. Thus, we have $R_T \le C' + \frac{\pi^2}{6} + \sum_{t=1}^{T}2\beta_t^{1/2}\sigma_{t-1}(x_t) + \sum_{t=T_0 +1}^{T} A(log(\frac{6}{\delta}))^{\frac{1}{d}}M_t$

To bound $\sum_{1}^{T}2\beta_t^{1/2}\sigma_{t-1}(x_t)$, we use the property of $\mathcal C_T$ and $\mathcal H_T$ that $\mathcal H_T \subseteq \mathcal X_T \subseteq \mathcal C_T$. Hence, similar to the proof of Lemma 5.4 of \citep{Srinivas12}, we have $\sum_{1}^{T}2\beta_t^{1/2}\sigma_{t-1}(x_t) \le \sqrt{C_1T\beta_T\gamma_T(\mathcal C_T)}$.
The remaining problem is to bound $\sum_{T_0}^{T} A(log(\frac{6}{\delta}))^{\frac{1}{d}}M_t = A(log(\frac{6}{\delta}))^{\frac{1}{d}}\sum_{T_0}^{T}M_t$, where
\begin{equation*}
  M_t=\begin{cases}
    (2 + ln(t))t^{-\frac{\lambda}{d}}, & \text{if $\alpha = -1$}.\\
   \frac{2}{\alpha +1}t^{-\frac{\lambda}{d}}, & \text{if $-1 < \alpha < 0$}.
  \end{cases}
\end{equation*}
We consider two cases of $\alpha$:
\begin{itemize}
  \item If $\alpha = -1$, then $\sum_{t=T_0}^{T}M_t \le \sum_{t=1}^{T} \frac{2 + ln(t)}{t^{\frac{\lambda}{d}}} < (2 + ln(T))\sum_{t=1}^{T} \frac{1}{t^{\frac{\lambda}{d}}}$. We consider three cases of $\lambda$:
\begin{itemize}
  \item if $\lambda = d$, $\sum_{t=1}^{T} \frac{1}{t^{\frac{\lambda}{d}}} = \sum_{t=1}^{T} \frac{1}{t} < 1 + ln(T)$ (using Lemma \ref{le:2}). Therefore, $\sum_{T_0}^{T} A(log(\frac{6}{\delta}))^{\frac{1}{d}}M_t < A(log(\frac{6}{\delta}))^{\frac{1}{d}}B_T$, where $B_T = (2 + ln(T))(1 + ln(T))$.
  \item if $\lambda > d$, $\sum_{t=1}^{T} \frac{1}{t^{\frac{\lambda}{d}}} < 1 + \frac{1}{\lambda/d -1} = \frac{\lambda}{\lambda - d} $(using Lemma \ref{le:3}). Thus, $\sum_{T_0}^{T} A(log(\frac{6}{\delta}))^{\frac{1}{d}}M_t < A(log(\frac{6}{\delta}))^{\frac{1}{d}}B_T$, where $B_T = (2 + ln(T))(1 + \frac{d}{d -\lambda})$.
  \item if $0 < \lambda < d$, $\sum_{t=1}^{T} \frac{1}{t^{\frac{\lambda}{d}}} < 1 + \frac{T^{1 - \frac{\lambda}{d}}}{1 - \frac{\lambda}{d}} < 1 + \frac{d}{d -\lambda}T^{1- \frac{\lambda}{d}}$. Thus, $\sum_{T_0}^{T} A(log(\frac{6}{\delta}))^{\frac{1}{d}}M_t < A(log(\frac{6}{\delta}))^{\frac{1}{d}}B_T$, where $B_T = (2 + ln(T))(1+ \frac{d}{d -\lambda}T^{1- \frac{\lambda}{d}})$.
\end{itemize}
  \item If $-1 < \alpha < 0$, then $\sum_{t=T_0}^{T}M_t \le \frac{2}{\alpha +1}(\sum_{t=1}^{T} \frac{1}{t^{\frac{\lambda}{d}}})$. Similar to the above case, we consider three cases of $\lambda$:
\begin{itemize}
   \item if $\lambda = d$, $\sum_{t=1}^{T} \frac{1}{t^{\frac{\lambda}{d}}} = \sum_{t=1}^{T} \frac{1}{t} < 1 + ln(T)$ (using Lemma \ref{le:2}). Therefore, $\sum_{T_0}^{T} A(log(\frac{6}{\delta}))^{\frac{1}{d}}M_t < A(log(\frac{6}{\delta}))^{\frac{1}{d}}B_T$, where $B_T = \frac{2}{\alpha +1}(1 + ln(T))$.
  \item if $\lambda > d$, $\sum_{t=1}^{T} \frac{1}{t^{\frac{\lambda}{d}}} < 1 + \frac{1}{\lambda/d -1} = \frac{\lambda}{\lambda - d} $(using Lemma \ref{le:3}). Thus, $\sum_{T_0}^{T} A(log(\frac{6}{\delta}))^{\frac{1}{d}}M_t < A(log(\frac{6}{\delta}))^{\frac{1}{d}}B_T$, where $B_T = \frac{2}{\alpha +1}(1 + \frac{d}{d -\lambda})$.
  \item if $0 < \lambda < d$, $\sum_{t=1}^{T} \frac{1}{t^{\frac{\lambda}{d}}} < 1 + \frac{T^{1 - \frac{\lambda}{d}}}{1 - \frac{\lambda}{d}} < 1 + \frac{d}{d -\lambda}T^{1- \frac{\lambda}{d}}$. Thus, $\sum_{T_0}^{T} A(log(\frac{6}{\delta}))^{\frac{1}{d}}M_t < A(log(\frac{6}{\delta}))^{\frac{1}{d}}B_T$, where $B_T = \frac{2}{\alpha +1}(1+ \frac{d}{d -\lambda}T^{1- \frac{\lambda}{d}})$.
\end{itemize}
\end{itemize}
For all cases,
with probability greater than $1 -\delta$ we achieve
$R_T \le  C'  +  \sqrt{C_1T\beta_T\gamma_T(\mathcal C_T)} + A(log(\frac{6}{\delta}))^{\frac{1}{d}}B_T + \frac{\pi^2}{6}$
, where $A =  s_2\sqrt{log(\frac{l_h ds_1}{\delta})}\frac{2(b-a)}{\pi}d\sqrt{d+2}$, and
$B_T =U_T V_T$ such that $U_T = 2 + ln(T)$ if $\alpha =-1$, otherwise $U_T = 2(\alpha +1)^{-1}$, and $V_T = 1 + ln(T)$ if $\lambda = d$, otherwise $V_T = 1 + \frac{d}{d- \lambda} \text{max}\{ 1, T^{1 -\frac{\lambda}{d}}\}$. Thus, Proposition 3 holds.
\begin{theorem}[Cumulative Regret $R_T$ of HD-HuBO Algorithm]
Let $f \sim \mathcal{GP}(\mathbf{0}, k)$ with a stationary covariance function $k$. Assume that there exist constants $s_1, s_2 > 0$ such that $\mathbb{P}[sup_{\mathbf{x} \in \mathcal X}|\partial f/ \partial x_{i}| > L] \le s_1e^{-(L/s_2)^2}$
for all $L > 0$ and for all $i \in \{1,2,..., d\}$. Pick a $\delta \in (0,1)$. Then, with $T > T_0$, under conditions $\lambda > d(\alpha +1)$, $-1 \le \alpha < 0$, $l_h > 0$, the cumulative regret of proposed HD-HuBO algorithm is bounded as
\begin{itemize}
  \item \scalebox{0.9}{$R_T \le \mathcal O^*(T^{\frac{(\alpha +1)d +1}{2}} + (log(\frac{6}{\delta}))^{\frac{1}{d}}B_T)$} if $k$ is a SE kernel,
  \item \scalebox{0.9}{$R_T \le \mathcal O^*(T^{\frac{d^2(\alpha +2) + d}{4\nu + 2d(d+1)}+ \frac{1}{2}} + (log(\frac{6}{\delta}))^{\frac{1}{d}}B_T)$} if $k$ is a Mat\'ern kernel,
\end{itemize}
with probability greater than $1 -\delta$, where $B_T =U_T V_T$ such that $U_T = 2 + ln(T)$ if $\alpha =-1$, otherwise $U_T = 2(\alpha +1)^{-1}$, and $V_T = 1 + ln(T)$ if $\lambda = d$, otherwise $V_T = 1 + \frac{d}{d- \lambda} \text{max}\{ 1, T^{1 -\frac{\lambda}{d}}\}$.
\label{theorem4}
\end{theorem}
\begin{proof}
Theorem holds due to Proposition 2 and Proposition 3.
\end{proof}
\section{Experiments}
\paragraph{On the initial search space}
\begin{figure}[ht]
  \centering
  \subfigure{\includegraphics[scale=1.0,width=.45\textwidth]{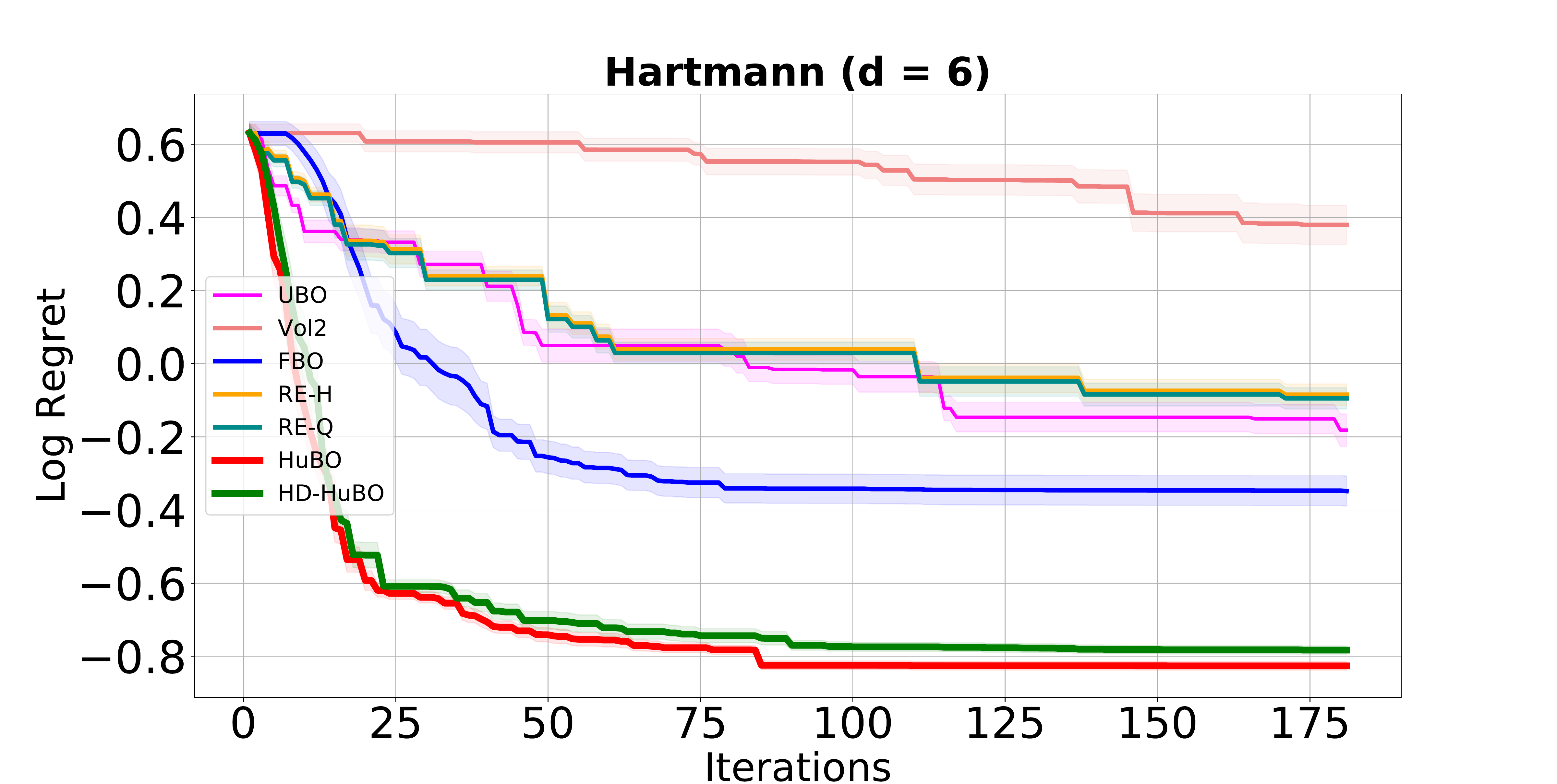}}\quad
  \subfigure{\includegraphics[scale=1.0,width=.45\textwidth]{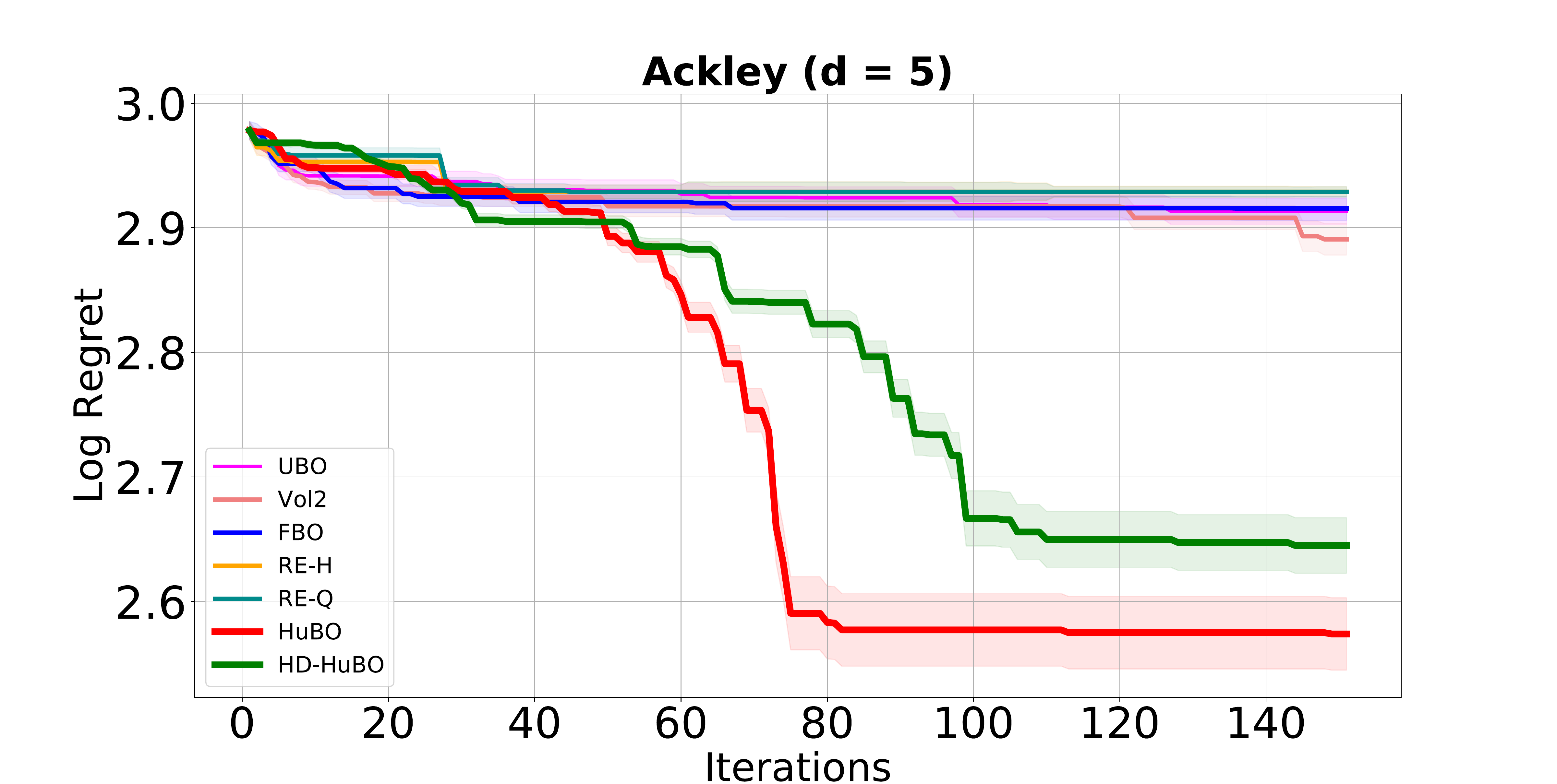}}
\caption{Comparison of baselines and the proposed methods when the initial search space is very small  fraction ($2\%$) of the pre-defined space.}
\label{initialSearchspace}
\end{figure}
The initial search space is crucial to the optimisation efficiency of any volume expansion strategy. However, since the search space is unknown, in reality it is possible that the initial search domain is very far from the global optimum. We consider this situation by setting the initial search space to be only $2\%$ of the pre-defined domain. Under this setting, we optimise two functions: Hartmann6  and 5-dims Ackley function. As seen in Figure \ref{initialSearchspace}, our algorithms outperform baselines due to the expansion and  especially translations of search spaces toward the promising regions. This is a benefit of our algorithm compared to the previous works in unknown search spaces.
\paragraph{On the computational effectiveness}
\begin{figure}[ht]
  \centering
  \includegraphics[scale=1.0,width=.45\textwidth]{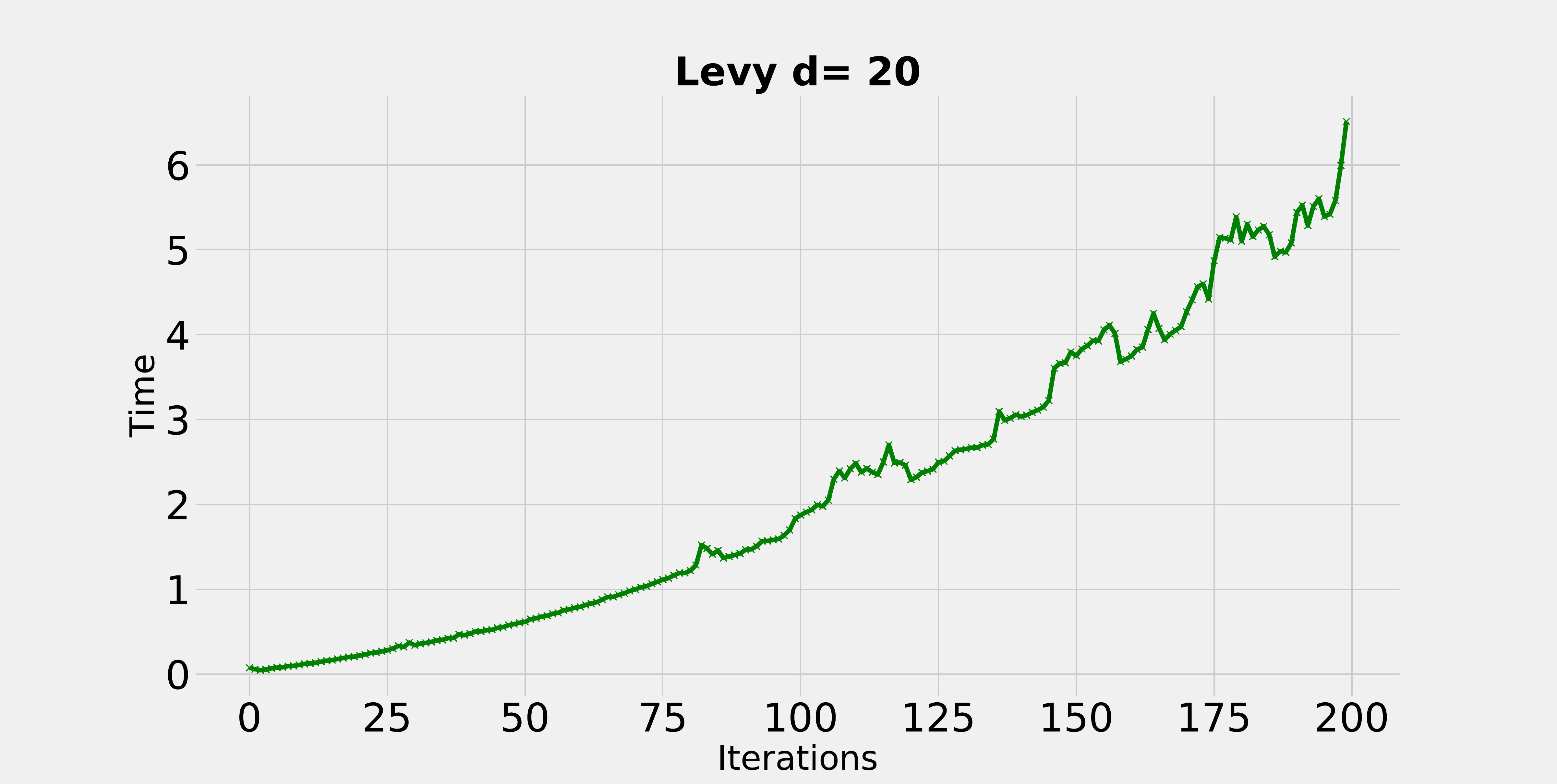}
\caption{The average runtime (seconds) of HD-HuBO over iterations.}
\label{running_time}
\end{figure}
\begin{table}[ht]
\centering
\caption{Average CPU time (seconds) at the final iteration for all algorithms.}
\begin{tabular}[t]{lccccc}
\hline
Algorithms & Beale & Hartmann3 & Hartmann6 & Levy(d =20) & Ackley(d =20)\\
\hline
HuBO           &  0.40  & 0.74   & 3.06    & 6.63 & 9.98  \\
HD-HuBO        &  0.48  & 0.76   & 3.14    & 6.90 & 11.13 \\
Re-H           &  0.49  & 0.84   & 0.91    & 7.22 & 12.96 \\
Re-Q           &  0.47  & 1.52   & 6.12    & 6.97 & 13.21 \\
Vol2           &  0.37  & 0.76   & 2.89    & 6.34 & 9.13  \\
UBO            &  0.61  & 2.11   & 11.21   & 9.37 & 21.33 \\
FBO            &  1.91  & 4.32   & 29.50   & 23.67& 46.56 \\
\hline
\end{tabular}
\end{table}%

The computational time is an important benefit for our algorithms. In our experiments, $\mathcal C_{initial}$ is set to 10 times to the size of the initial search space $\mathcal X_0$ along each dimension, it allows expanded spaces to move freely to any position in the pre-defined domain. It follows that via the transformation, the center of the new search space is set closer to the best solution found up to that iteration. Therefore, both the new bound and the new center are easy to determine compared to previous works in unknown search spaces except the volume doubling strategy. We note that in practice, if the search domain is unknown, our algorithm would typically benefit by setting a large $\mathcal C_{initial}$ as this allows the search space to be
centered close to the best found solution.

For HD-HuBO, to optimise over multiple disjoint hypercubes in the continuous input space, we perform optimisation for each hypercube and then take the best maximum value found across all hypercubes. For example, for synthetic functions we used $\lambda = 1, N_0 = 1$ and thus $N_t = t$. This means that at iteration $t$, we use $t$ hypercubes for the maximisation of acquisition function. We optimise the
acquisition function using L-BFGS with 20 restarts on each hypercube. The maximum number of acquisition function evaluations is set to 1000. The Figure \ref{running_time} shows the average runtime (seconds) of HD-HuBO over iterations on the 20-dims Levy function.

To compare the computational time of all algorithms, we give to all the algorithms the equal computational budget to maximise acquisition functions at each iteration. As seen in Table 1, our algorithms are faster than UBO which needs to compute singular values of matrix $(\textbf{K} + \sigma \textbf{I})^{-1}$, and faster than FBO, which needs extra steps to numerically solve multiple optimisation problems for FBO.
\paragraph{Additional Results}
\begin{figure}[ht]
  \centering
  \subfigure{\includegraphics[scale=1.0,width=.45\textwidth]{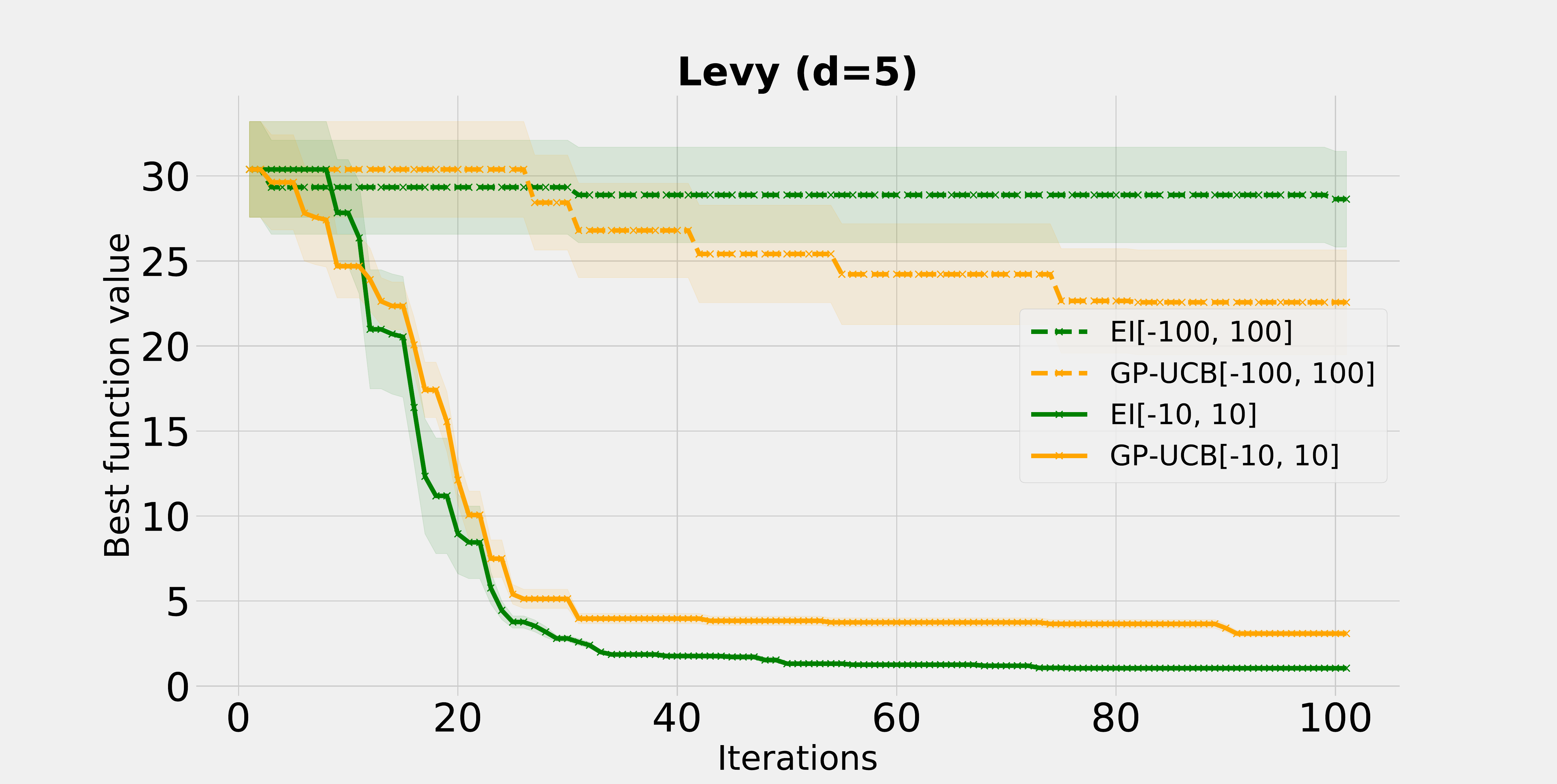}}\quad
  \subfigure{\includegraphics[scale=1.0,width=.45\textwidth]{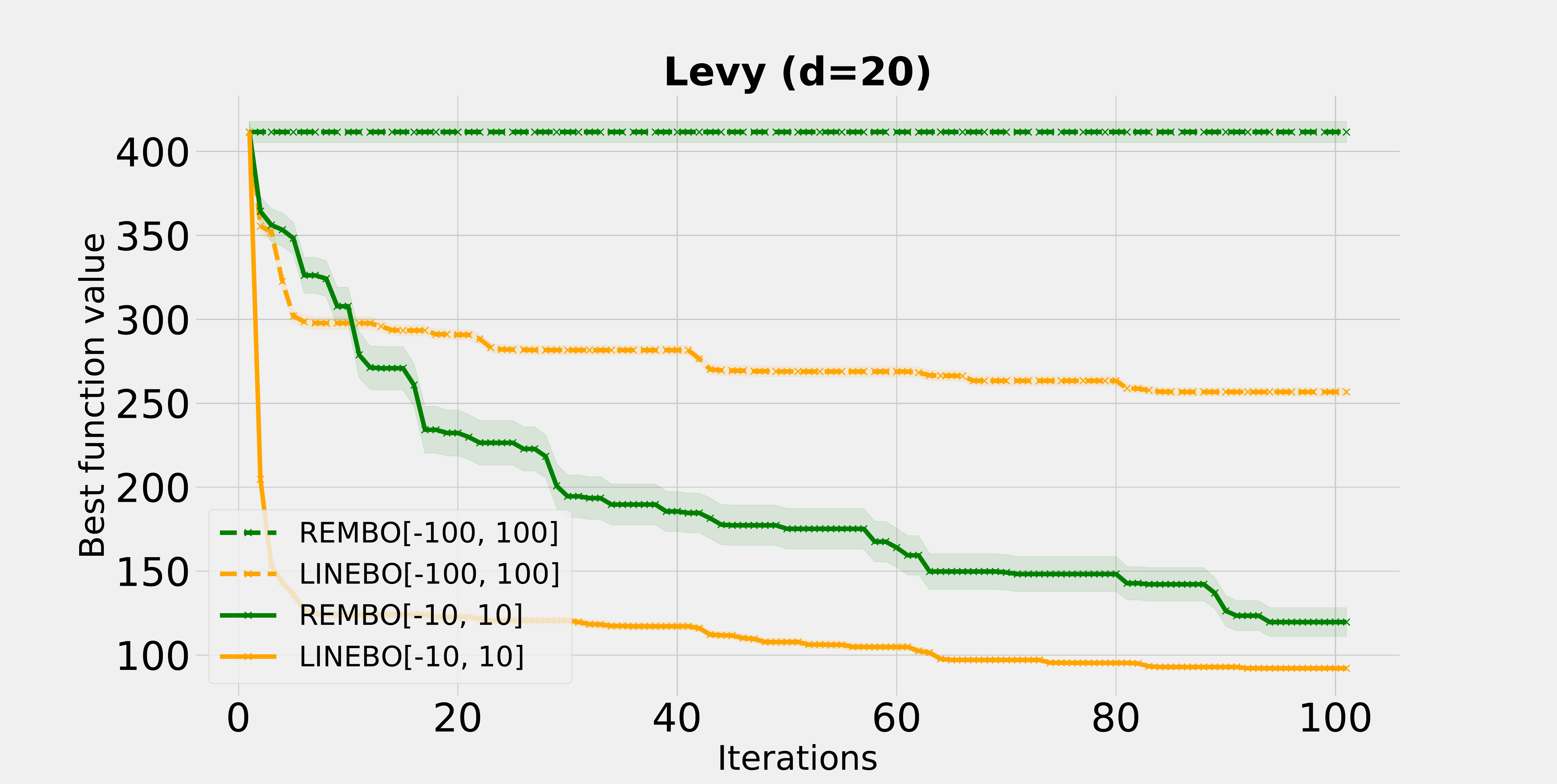}}
\caption{Optimisation efficiency with different sizes of the search space}
\label{volume_search_space}
\end{figure}
When the search space is \emph{unknown}, one heuristic solution is to specify it arbitrarily. However, there are two problems: (1) an arbitrary search space that is finite, no matter how large, may not contain the global optimum (2) optimisation efficiency decreases with increasing size of the search space. We below provide two examples to illustrate that the optimisation efficiency decreases with increasing size of the search space.

In low dimensions, we consider the optimisation efficiency of BO algorithms such as EI and GP-UCB on 5-dims Levy function when increasing the size of the search space. We consider two cases: (1) the search space is set to $[-10, 10]$ and (2) the search space is set to $[-100, 100]$. In high dimensions, we consider the optimisation efficiency of REMBO algorithm \citep{Wang13} and LINEBO algorithm \citep{Johannes19} on 20-dims Levy function. Also, we consider two cases: (1) the search space is set to $[-10, 10]$ and (2) the search space is set to $[-100, 100]$. The Levy function achieves the minimum value at $x^* = (1,1,...,1)$.  As seen in Figure \ref{volume_search_space}, the use of a larger space slows down fast the convergence. In contrast, our approach using a volume expansion strategy starting from a small initial search space can avoid this unnecessary sampling.
\end{document}